%% file: main.tex
\documentclass[letterpaper]{article} 
\usepackage{aaai2026}  
\usepackage{times}  
\usepackage{helvet}  
\usepackage{courier}  
\usepackage[hyphens]{url}  
\usepackage{graphicx} 
\usepackage{natbib}  
\usepackage{caption} 
\usepackage{algorithm}
\usepackage{algorithmic}
\usepackage{amsmath}
\usepackage{booktabs}
\usepackage{amsthm}
\usepackage{amssymb}  
\usepackage{multirow} 
\newtheorem{theorem}{Theorem}[]
\newtheorem{assumption}{Assumption}[]
\newtheorem{lemma}{Lemma}[]

\theoremstyle{definition}                  

\usepackage{newfloat}
\usepackage{listings}
\DeclareCaptionStyle{ruled}{labelfont=normalfont,labelsep=colon,strut=off} 
\floatstyle{ruled}
\newfloat{listing}{tb}{lst}{}
\floatname{listing}{Listing}
\title{Renormalization Group Guided Tensor Network Structure Search}
\author{
    Maolin Wang\textsuperscript{\rm 1},
    Bowen Yu\textsuperscript{\rm 1},
    Sheng Zhang\textsuperscript{\rm 1},
    Linjie Mi\textsuperscript{\rm 1},
    Wanyu Wang\textsuperscript{\rm 1},\\
    Yiqi Wang\textsuperscript{\rm 2},
    Pengyue Jia\textsuperscript{\rm 1},
    Xuetao Wei\textsuperscript{\rm 3},
    Zenglin Xu\textsuperscript{\rm 4},
    Ruocheng Guo\textsuperscript{\rm 5},
    Xiangyu Zhao\textsuperscript{\rm 1}\thanks{Corresponding author.}
}
\affiliations{
    \textsuperscript{\rm 1}City University of Hong Kong\\
        \textsuperscript{\rm 2} National University of Defense Technology\\
    \textsuperscript{\rm 3} Southern University of Science and Technology\\
    \textsuperscript{\rm 4} Fudan University\\
    \textsuperscript{\rm 5}Intuit AI\\
    \{morin.wang, yubowen,szhang844-c,linjiemi2-c,wanyuwang4-c,jia.pengyue\} @my.cityu.edu.hk \\ yiq@nudt.edu.cn,  weixt@sustech.edu.cn, \{zenglin, rguo.asu\}@gmail.com, xianzhao@cityu.edu.hk
}

\usepackage{bibentry}

\begin{document}

\maketitle

\begin{abstract}
Tensor network structure search (TN-SS) aims to automatically discover optimal network topologies and rank configurations for efficient tensor decomposition in high-dimensional data representation. Despite recent advances, existing TN-SS methods face significant limitations in computational tractability, structure adaptivity, and optimization robustness across diverse tensor characteristics. They struggle with three key challenges: single-scale optimization missing multi-scale structures, discrete search spaces hindering smooth structure evolution, and separated structure-parameter optimization causing computational inefficiency.
We propose RGTN (Renormalization Group guided Tensor Network search), a physics-inspired framework transforming TN-SS via multi-scale renormalization group flows. Unlike fixed-scale discrete search methods, RGTN uses dynamic scale-transformation for continuous structure evolution across resolutions. Its core innovation includes learnable edge gates for optimization-stage topology modification and intelligent proposals based on physical quantities like node tension measuring local stress and edge information flow quantifying connectivity importance. Starting from low-complexity coarse scales and refining to finer ones, RGTN finds compact structures while escaping local minima via scale-induced perturbations.
Extensive experiments on light field data, high-order synthetic tensors, and video completion tasks show RGTN achieves state-of-the-art compression ratios and runs 4-600$\times$ faster than existing methods, validating the effectiveness of our physics-inspired approach.
\end{abstract}

\begin{links}
    \link{Code}{https://github.com/Applied-Machine-Learning-Lab/RGTN}
    \link{Appendix}{https://github.com/Applied-Machine-Learning-Lab/RGTN/Appendix.pdf}
\end{links}
\input{1Introduction}
\input{2Method}

\input{3Experiments}

\input{4RelatedWork}
\input{5Conclusion}

\bibliography{aaai2026}
\input{Appendix}
\end{document}

%% file: 1Introduction.tex
\section{Introduction}

Tensor network structure search (TN-SS) has emerged as a fundamental challenge in high-dimensional data representation, seeking to automatically discover optimal network topologies and rank configurations for efficient tensor decomposition~\cite{li2020evolutionary}. Despite recent advances, existing TN-SS methods face significant limitations in effectively addressing three critical aspects: computational tractability~\cite{wang2023tensor}, structure adaptivity~\cite{hashemizadeh2020adaptive}, and optimization robustness across diverse tensor characteristics~\cite{iacovides2025domain}. While current approaches show promise in specific scenarios, they struggle to comprehensively tackle these interconnected challenges~\cite{zheng2020ansor}.

In practical tensor decomposition scenarios, optimal network structures naturally exhibit three fundamental properties deeply rooted in physics and information theory: (1)~\textbf{Scale-Invariant Correlations}: tensor networks possess self-similar correlation structures across different length scales, analogous to critical phenomena in statistical physics where the renormalization group reveals how physical properties transform under scale changes~\cite{white1992density,vidal2007entanglement}, (2)~\textbf{Hierarchical Entanglement}: the optimal connectivity pattern reflects hierarchical entanglement structures, with different scales capturing correlations at different ranges—from local quantum entanglement to global classical correlations~\cite{schollwock2011density,orus2014practical}, and (3)~\textbf{Flow of Information}: efficient tensor networks naturally organize information flow from fine-grained local features to coarse-grained global structures, following principles similar to real-space renormalization in condensed matter physics~\cite{evenbly2011tensor,haegeman2013entanglement}. These observations from quantum many-body physics and renormalization group theory highlight the critical need for tensor network methods that can exploit multi-scale structures while maintaining  efficiency~~\cite{chan2008introduction}.

\begin{figure*}[t!]
    \centering
    \includegraphics[width=0.9\textwidth]{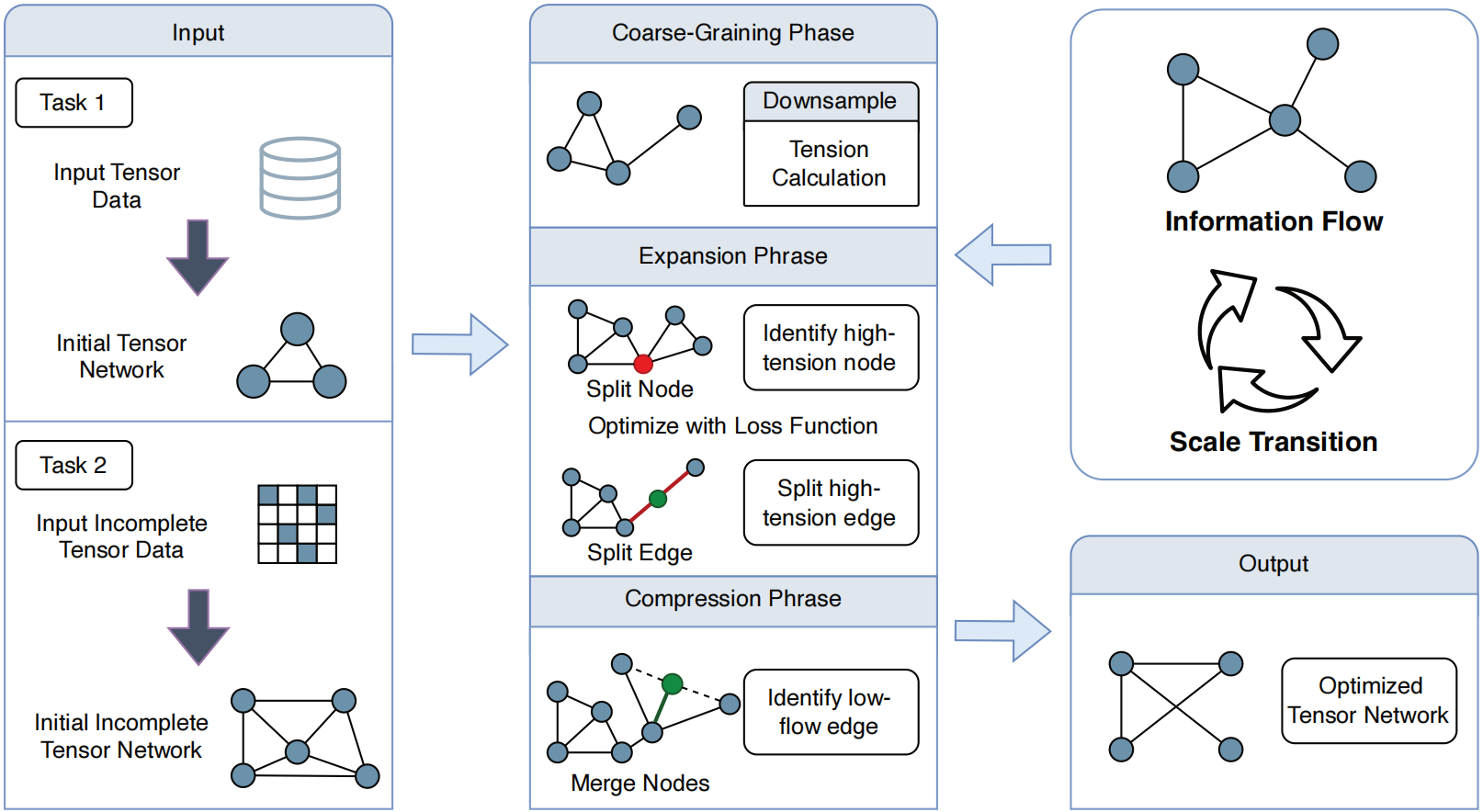}
    \caption{The RGTN framework transforms network topology through physics-inspired multi-scale operations instead of traditional sampling-evaluation methods. It processes tensors through three RG-based phases: coarse-graining (downsampling and tension calculation), expansion (splitting high-tension nodes), and compression (merging low-flow edges), unifying structure search and parameter optimization for efficient discovery of optimal tensor network structures.}
    \label{fig:rgtn_overview}
\end{figure*}

Based on these observations, tensor network structure search systems need to address three fundamental challenges that require innovative solutions:
First, single-scale optimization inherently limits structure discovery. Current methods operate at a fixed resolution throughout optimization, missing the rich multi-scale structures inherent in tensor data and failing to leverage the computational advantages of hierarchical decomposition~\cite{li2020evolutionary,hashemizadeh2020adaptive,li2023alternating}.
Second, discrete search spaces prevent smooth structure evolution. Genetic algorithms like TNGA explore topology through discrete mutations~\cite{li2020evolutionary}, greedy methods incrementally modify structures through local decisions~\cite{hashemizadeh2020adaptive}, and local search approaches like TNLS navigate neighborhoods of current solutions~\cite{li2023alternating}, all suffering from the combinatorial nature of the search space~\cite{li2023alternating}.
Third, separation of structure and parameter optimization creates inefficiency. Program synthesis methods intelligently generate candidate structures but still require expensive evaluation of each proposal~\cite{zheng2020ansor,liao2019differentiable}, while regularization-based approaches achieve faster convergence but remain constrained by predetermined topology spaces~\cite{zheng2024svdinstn}.
These limitations manifest critically in practice: Evolutionary methods require population sizes growing exponentially with tensor order~\cite{li2020evolutionary}, greedy algorithms make irrevocable local decisions that lead to suboptimal structures~\cite{hashemizadeh2020adaptive}, local search methods frequently converge to poor local minima~\cite{li2022permutation,li2023alternating}, and even advanced approaches using LLMs for algorithm discovery still operate within the sampling-evaluation paradigm~\cite{zheng2020ansor}. This creates an urgent need for a fundamentally new paradigm that can harness physics principles for efficient structure discovery.

To address these challenges, we propose RGTN (\textbf{R}enormalization \textbf{G}roup guided \textbf{T}ensor \textbf{N}etwork search), a physics-inspired framework that transforms tensor network structure search through multi-scale renormalization group flows. Unlike existing methods limited to discrete structure spaces, RGTN implements dynamic scale transformation where networks evolve continuously across resolution levels via learnable edge gates. This approach utilizes node tension to measure local stress and edge information flow to quantify connectivity importance. By optimizing from coarse to fine scales, RGTN discovers compact structures while escaping local minima through scale-induced perturbations.

Our main contributions are:
\begin{itemize}
\item \textbf{Multi-scale framework}: First tensor network approach implementing true renormalization group flows with continuous edge gates and scale-dependent optimization, enabling dynamic topology evolution beyond discrete search limitations.

\item \textbf{Physics-inspired strategies}: Node tension and edge information flow guide intelligent structure modifications through natural physical processes rather than combinatorial enumeration.

\item \textbf{Theoretical speedup}: Rigorous analysis showing exponential acceleration from $\Omega(\exp(N^2))$ to $\mathcal{O}(\log I \cdot \log(1/\epsilon))$ with stronger convergence guarantees and high-probability escape from local minima.

\item \textbf{Empirical validation}: Experiments demonstrate RGTN achieves up to 3× better compression and 4-600× faster over existing methods.
\end{itemize}

%% file: 2Method.tex
\section{Method}

In this section, we present our RGTN approach for efficiently searching tensor network structures. 
As shown in Figure~\ref{fig:rgtn_overview}, we propose a radically different approach inspired by the renormalization group (RG) theory~\cite{shankar1994renormalization,ueda2024renormalization} from statistical physics. 

\subsubsection{Theoretical Foundation}

The renormalization group is a mathematical apparatus that reveals how physical systems behave across different length scales. In the context of tensor networks, we interpret scale transformations as changes in the network's ability to capture correlations at different ranges. Consider a tensor network $\mathcal{M}$ representing a tensor $\mathcal{X} \in \mathbb{R}^{I_1 \times I_2 \times \cdots \times I_N}$. We define a renormalization group flow on the space of tensor networks through a semi-group of transformations $\{R_s\}_{s \geq 0}$:

\begin{equation}
\mathcal{M}_{s+\delta s} = R_{\delta s}[\mathcal{M}_s],
\end{equation}

where $s$ represents the scale parameter. The RG transformation $R_s$ consists of two complementary operations that modify the network structure while preserving its representational capacity.

The expansion operation $R_{\text{expand}}$ increases the network's resolution by decomposing tensor cores:
\begin{equation}
R_{\text{expand}}: \mathcal{G}_v \mapsto \sum_{r=1}^{R_{uv}} \mathcal{G}_u^{(r)} \otimes \mathcal{G}_v^{(r)},
\end{equation}
where a single core $\mathcal{G}_v$ is split into two cores $\mathcal{G}_u$ and $\mathcal{G}_v$ connected by a bond of dimension $R_{uv}$. This operation enables the network to capture finer-grained correlations.

The compression operation $R_{\text{compress}}$ reduces the network's complexity by merging adjacent cores:
\begin{equation}
R_{\text{compress}}: (\mathcal{G}_u, \mathcal{G}_v) \mapsto \mathcal{G}_{uv} = \mathcal{G}_u \times_{e_{uv}} \mathcal{G}_v,
\end{equation}
where $\times_{e_{uv}}$ denotes tensor contraction along the edge connecting $u$ and $v$. This operation identifies and eliminates redundant degrees of freedom.

\subsubsection{Scale-Dependent Effective Action}

Following the RG philosophy, we introduce a scale-dependent effective action (loss function) that captures the relevant physics at each scale:
\begin{equation}
S_s[\mathcal{M}] = S_{\text{data}}[\mathcal{M}] + \sum_{k} \lambda_k(s) S_k[\mathcal{M}],
\end{equation}
where $S_{\text{data}}$ represents data fidelity and $S_k$ are regularization terms with scale-dependent coupling constants $\lambda_k(s)$. The running of these coupling constants is determined by the RG flow equations:
\begin{equation}
\frac{d\lambda_k}{ds} = \beta_k(\{\lambda_j\}),
\end{equation}
where $\beta_k$ are the beta functions encoding how different regularization strengths evolve across scales.

\subsection{Enhanced Tensor Network Architecture}
Standard tensor network architectures are rigid, with fixed topologies that cannot adapt during optimization. This limitation prevents the network from discovering more efficient structures or adjusting its capacity based on the data complexity. Additionally, existing methods for structure modification require discrete decisions (add/remove edges, change ranks) that disrupt the optimization process and often lead to instability. We need architectural components that enable continuous structure adaptation while maintaining stable gradient flow throughout the network.

\subsubsection{Adaptive Diagonal Factors}

Inspired by the SVDinsTN's use of diagonal factors for structure discovery, we introduce adaptive diagonal factors that serve as importance weights for each tensor core. For a tensor core $\mathcal{G}_k \in \mathbb{R}^{R_1 \times \cdots \times R_m \times I_k}$, we define diagonal adaptation matrices $\mathbf{D}_k^{(i)} \in \mathbb{R}^{R_i \times R_i}$ for each virtual bond:
\begin{equation}
\tilde{\mathcal{G}}_k = \mathcal{G}_k \times_1 \mathbf{D}_k^{(1)} \times_2 \mathbf{D}_k^{(2)} \cdots \times_m \mathbf{D}_k^{(m)}.
\end{equation}

These diagonal factors play a crucial role in structure discovery. When elements of $\mathbf{D}_k^{(i)}$ approach zero, the corresponding bond dimensions become effectively reduced, automatically revealing a more compact structure.

\subsubsection{Edge Gating Mechanism}

To enable dynamic topology modification during optimization, we introduce learnable edge gates. For each edge $(u,v)$ in the tensor network graph $G = (V, E)$, we define a gating function:
\begin{equation}
g_{uv} = \sigma(w_{uv}), \quad w_{uv} \in \mathbb{R},
\end{equation}
where $\sigma$ is the sigmoid function. The gated tensor contraction becomes:
\begin{equation}
\mathcal{C}_{uv} = g_{uv} \cdot (\mathcal{G}_u \times_{m_u^v, m_v^u} \mathcal{G}_v) + (1-g_{uv}) \cdot \mathcal{I},
\end{equation}
where $\mathcal{I}$ represents an identity-like tensor maintaining dimensional consistency. This soft gating mechanism allows gradual edge removal when $g_{uv} \to 0$.

\subsubsection{Multi-Scale Loss Function}

For tensor completion, we formulate a comprehensive loss function that incorporates both data fidelity and structure-inducing regularization:
\begin{align}
\mathcal{L}_{\text{total}} = &\underbrace{\frac{1}{2}\|\mathcal{P}_{\Omega}(\mathcal{X} - \mathcal{F})\|_F^2}_{\text{Data Fidelity}} + \underbrace{\alpha(s) \sum_{t=1}^{T-1} \|\mathcal{X}^{(t+1)} - \mathcal{X}^{(t)}\|_F}_{\text{Temporal Consistency}} \nonumber \\
&+ \underbrace{\beta(s) \sum_{i,j} \|\nabla_{ij}\mathcal{X}\|_F}_{\text{Spatial Smoothness}} + \underbrace{\gamma(s) \sum_{k,i} \|\mathbf{D}_k^{(i)}\|_1}_{\text{Diagonal Sparsity}} \nonumber \\
&+ \underbrace{\delta(s) \sum_{(u,v) \in E} H(g_{uv})}_{\text{Edge Entropy}} + \underbrace{\epsilon(s) \text{TNN}(\mathcal{X})}_{\text{Low-rank Regularization}},
\end{align}
where $H(g) = -g\log g - (1-g)\log(1-g)$ is the binary entropy function encouraging decisive gating, and $\text{TNN}(\mathcal{X})$ is the tensor nuclear norm computed as:
\begin{equation}
\text{TNN}(\mathcal{X}) = \sum_{k=1}^{N} \omega_k \|\mathcal{X}_{(k)}\|_*,
\end{equation}
with $\mathcal{X}_{(k)}$ being mode-$k$ unfolding and $\|\cdot\|_*$ nuclear norm.

\subsection{Intelligent Structure Search via RG Flow}

Random or exhaustive structure search strategies suffer from poor scalability and often explore irrelevant regions of the structure space. Without guidance from the current network state, these methods waste computational resources on unpromising structural modifications. Furthermore, the interplay between structure optimization and parameter optimization is poorly understood in existing approaches, leading to suboptimal coordination between these two aspects. We need an intelligent search strategy that leverages the current network's properties to guide exploration and properly balances structural and parametric updates.

\subsubsection{Smart Proposal Generation}

Rather than random structural modifications, we use the current network's properties to guide proposal generation. For the expansion phase, we identify nodes with high "tension" - a measure of how much a node contributes to the reconstruction error:
\begin{equation}
T_v = \left\|\frac{\partial \mathcal{L}_{\text{data}}}{\partial \mathcal{G}_v}\right\|_F \cdot \text{degree}(v).
\end{equation}
Nodes with high tension are prioritized for splitting, as they likely encode complex correlations that benefit from finer representation.

For the compression phase, we identify edges with low "information flow" - quantified by the gate values and the mutual information between connected cores:
\begin{equation}
I_{uv} = g_{uv} \cdot \text{MI}(\mathcal{G}_u, \mathcal{G}_v),
\end{equation}
where $\text{MI}$ denotes mutual information estimated through the singular value spectrum of the contracted tensor.

\subsubsection{Adaptive Optimization Strategy}

The optimization of tensor cores and structural parameters proceeds through an adaptive scheme that adjusts to the current scale and convergence behavior. We employ a modified Adam optimizer with scale-dependent learning rates:
\begin{equation}
\eta_{\text{cores}}(s) = \eta_0 \cdot \exp(-s/s_0), \quad \eta_{\text{struct}}(s) = \eta_0 \cdot (1 + s/s_1),
\end{equation}
where cores are optimized more aggressively at fine scales while structural parameters are refined more at coarse scales.

The complete algorithm proceeds at \textbf{Appendix A}.

\subsubsection{Structure Discovery through Sparsity}

The interplay between diagonal factors and edge gates enables automatic structure discovery. During optimization, the $\ell_1$ regularization on diagonal factors and entropy regularization on edge gates induce sparsity patterns that reveal the underlying structure. Specifically, when diagonal factor elements $D_k^{(i)}[j,j] < \epsilon$, the corresponding bond dimension can be reduced, and when edge gate $g_{uv} < \delta$, the edge can be removed from the topology.

This soft-to-hard thresholding strategy is implemented through a temperature annealing scheme:
\begin{equation}
\tau(t) = \tau_0 \exp(-t/t_0),
\end{equation}
where the soft gates $g_{uv} = \sigma(w_{uv}/\tau(t))$ become increasingly binary as training progresses.

\subsection{Multi-Scale Progressive Refinement}

Direct optimization of large-scale tensor networks faces severe challenges, including slow convergence, susceptibility to local minima, and high computational cost. Starting from random initialization often requires extensive iterations to reach good solutions, and the optimization landscape becomes increasingly complex with network size. Additionally, fine-scale details can obscure the global structure, making it difficult to identify the optimal topology. We need a multi-scale approach that can efficiently explore the solution space by solving progressively refined versions of the problem.
 We begin at the coarsest scale, where the problem has reduced dimensionality and computational cost:
\begin{equation}
\mathcal{F}_S = \mathcal{D}_S[\mathcal{F}], \quad \text{where } \mathcal{D}_s \text{ is a downsampling operator}.
\end{equation}

As we flow towards finer scales, we use the coarse-scale solution to initialize the fine-scale optimization:
\begin{equation}
\mathcal{M}_{s-1}^{(0)} = \mathcal{U}_s[\mathcal{M}_s^*], \quad \text{where } \mathcal{U}_s \text{ is an upsampling operator}.
\end{equation}

This progressive refinement strategy provides several benefits: (1) faster convergence by providing good initializations, (2) avoiding local minima by exploring the solution space hierarchically, and (3) computational efficiency by solving smaller problems first.

\section{Theoretical Analysis}

This section provides theoretical analysis of the RGTN framework, establishing convergence guarantees, analyzing structure discovery properties, demonstrating computational advantages, and connecting to statistical physics principles.
Due to space constraints, detailed theorems on computational complexity, loss landscape smoothing, probabilistic escape from local minima, fixed points, criticality, and universality are in  \textbf{Appendix B}.
\subsection{Preliminaries and Assumptions}

We begin by establishing the mathematical foundations and assumptions underlying our analysis. Let $\mathcal{M} = (\mathcal{G}, \mathcal{S})$ denote a tensor network with cores $\mathcal{G} = \{\mathcal{G}_1, \ldots, \mathcal{G}_N\}$ and structure $\mathcal{S} = (V, E, \{R_{ij}\})$, where $V$ is the set of nodes, $E$ is the set of edges, and $R_{ij}$ are bond dimensions. The parameter space is denoted as $\Theta = \{\theta_{\mathcal{G}}, \theta_{\mathcal{S}}\}$, where $\theta_{\mathcal{G}}$ represents tensor core parameters and $\theta_{\mathcal{S}}$ represents structural parameters including diagonal factors and edge gates.

\begin{assumption}[Lipschitz Continuity]
\label{assum:lipschitz}
The loss function $\mathcal{L}(\mathcal{M})$ is $L$-Lipschitz continuous with respect to the network parameters:
\begin{equation}
|\mathcal{L}(\mathcal{M}_1) - \mathcal{L}(\mathcal{M}_2)| \leq L \|\mathcal{M}_1 - \mathcal{M}_2\|_F,
\end{equation}
where $\|\cdot\|_F$ denotes the Frobenius norm extended to tensor networks.
\end{assumption}

\begin{assumption}[Smoothness of Scale Transformations]
\label{assum:scale_smooth}
The scale transformation operators $\mathcal{D}_s$ (downsampling) and $\mathcal{U}_s$ (upsampling) satisfy:
\begin{equation}
\|\mathcal{U}_s \circ \mathcal{D}_s[\mathcal{X}] - \mathcal{X}\|_F \leq C_s \|\mathcal{X}\|_F,
\end{equation}
where $C_s = \mathcal{O}(2^{-s})$ decreases with coarser scales.
\end{assumption}

\begin{assumption}[Bounded Network Parameters]
\label{assum:bounded_params}
The network parameters lie in a bounded domain: $\|\mathcal{G}_k\|_F \leq B_{\mathcal{G}}$ for all cores and $0 \leq g_{ij} \leq 1$ for all edge gates.
\end{assumption}

\begin{table*}[t!]
    \centering
    \graphicspath{{figures/}}
    \begin{tabular}{l ccccc}
        \toprule
        
        \textbf{True structure (4th-order)} &
        \includegraphics[width=0.12\textwidth]{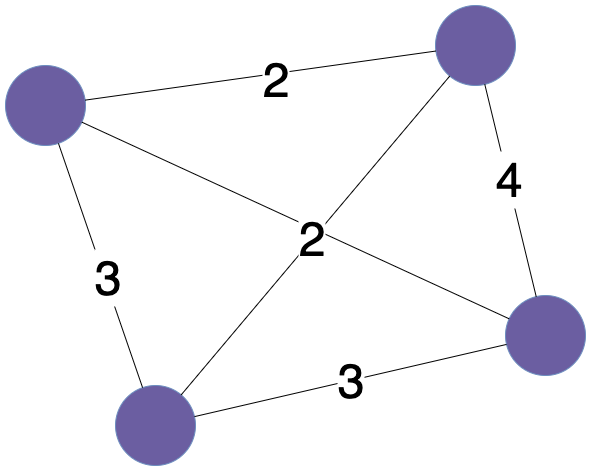} &
        \includegraphics[width=0.12\textwidth]{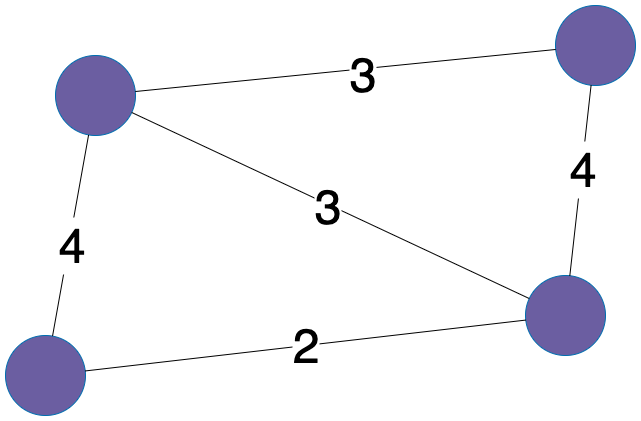} &
        \includegraphics[width=0.12\textwidth]{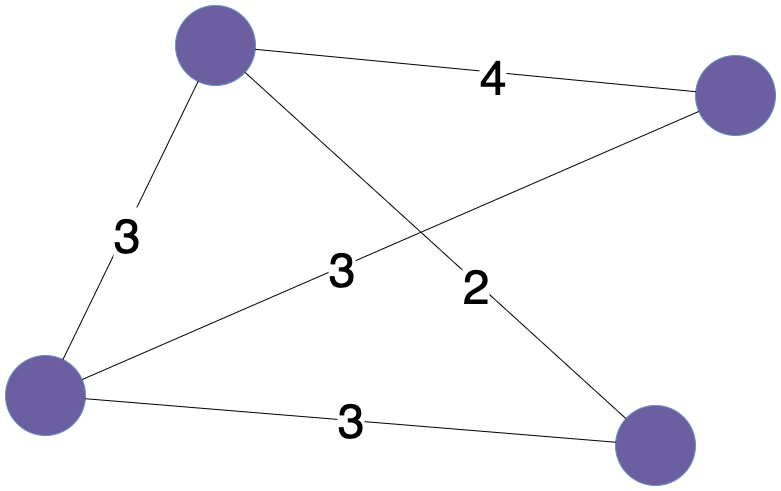} &
        \includegraphics[width=0.12\textwidth]{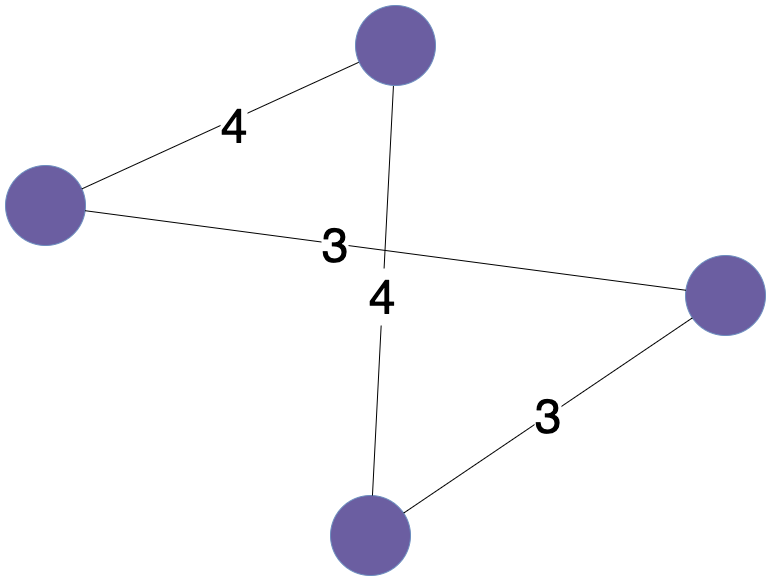} &
        \includegraphics[width=0.12\textwidth]{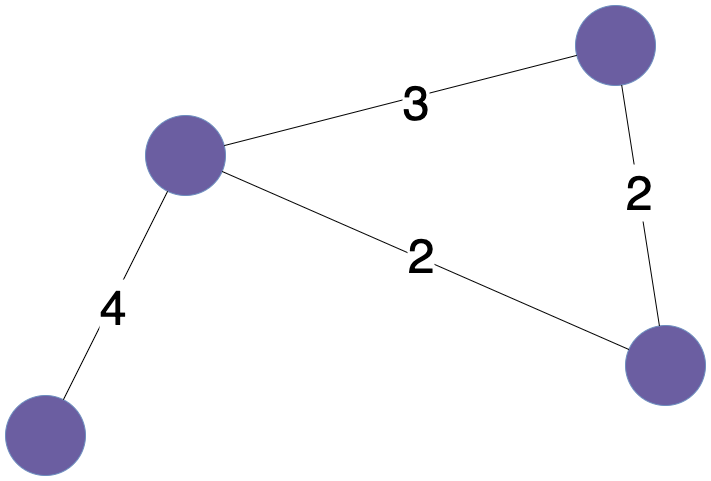} \\
        
        \midrule

        \textbf{Success rate} & 100\% & 100\% & 96\% & 95\% & 99\% \\
        
        \midrule[\heavyrulewidth]
        
        \textbf{True structure (5th-order)} &
        \includegraphics[width=0.12\textwidth]{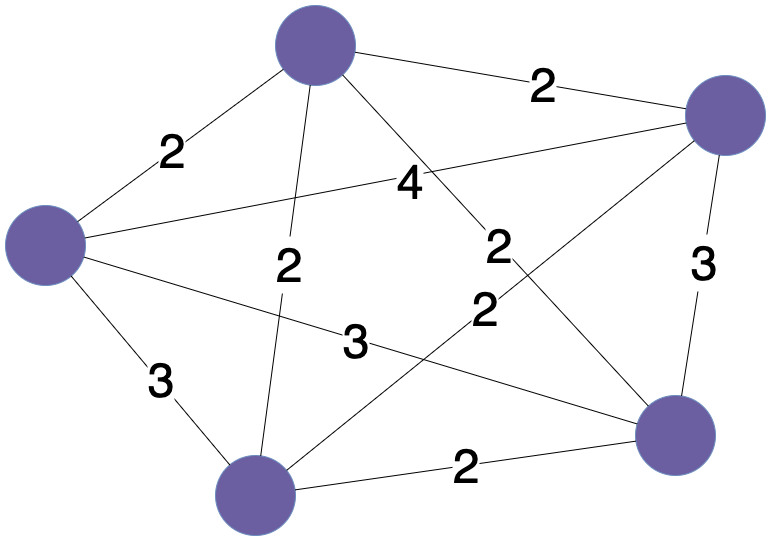} &
        \includegraphics[width=0.12\textwidth]{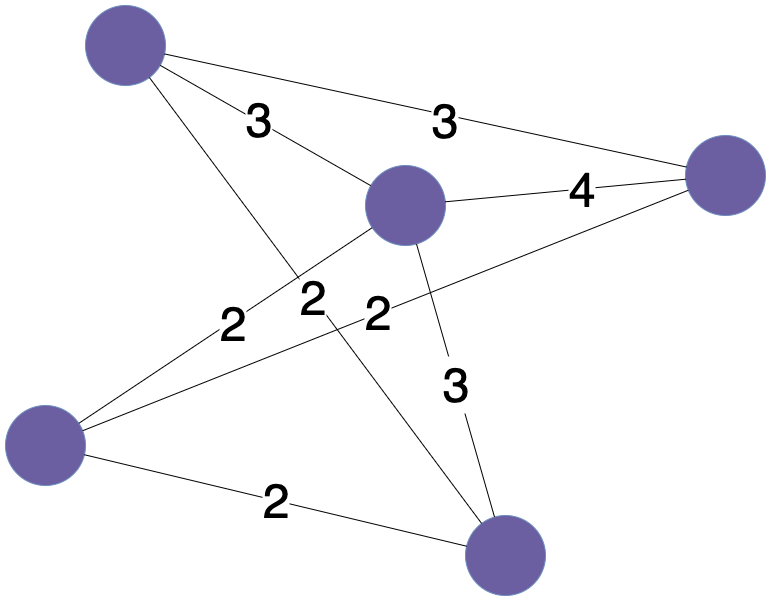} &
        \includegraphics[width=0.12\textwidth]{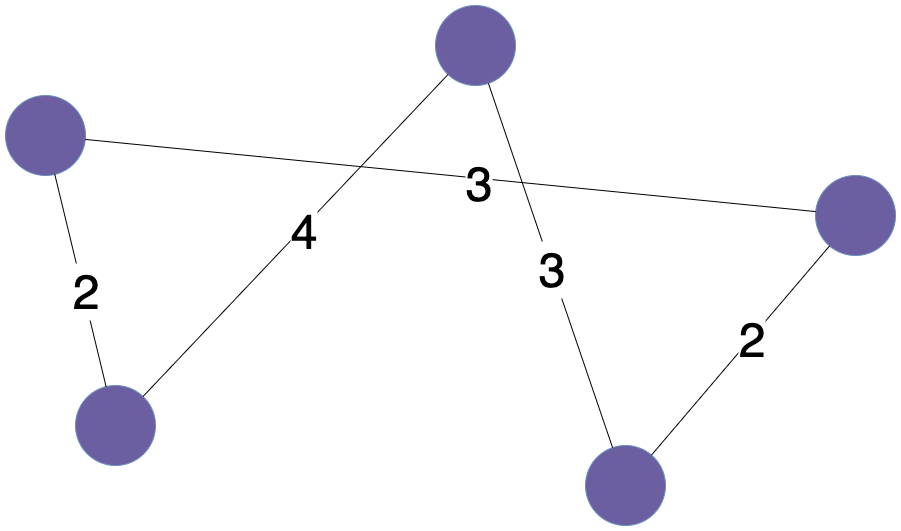} &
        \includegraphics[width=0.12\textwidth]{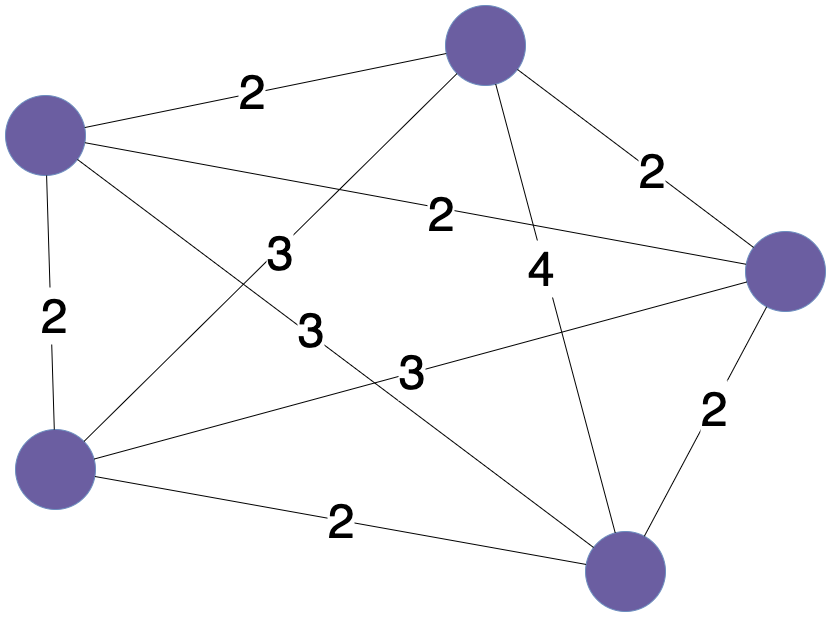} &
        \includegraphics[width=0.12\textwidth]{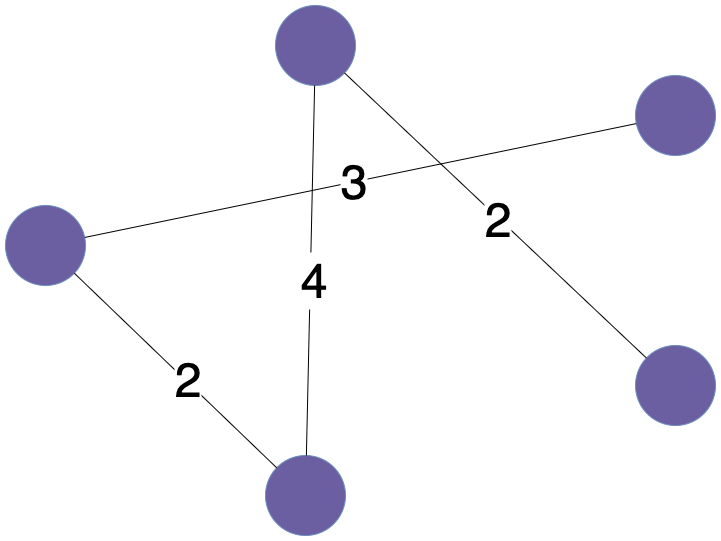} \\

        \midrule

        \textbf{Success rate} & 100\% & 98\% & 96\% & 97\% & 100\% \\
        
        \bottomrule
    \end{tabular}
    \caption{Performance of RGTN on TN structure revealing under 100 independent tests.}
    \label{tab:my_performance_table}
\end{table*}

\subsection{Convergence Analysis}

We first establish the convergence properties of the RGTN algorithm. The analysis considers the alternating optimization between tensor cores and structural parameters across multiple scales.

\begin{theorem}[Global Convergence of Multi-Scale Optimization]
\label{thm:global_convergence}
Under Assumptions \ref{assum:lipschitz}-\ref{assum:bounded_params}, the RGTN algorithm generates a sequence of networks $\{\mathcal{M}^{(t)}\}_{t=0}^{\infty}$ that converges to a critical point of the multi-scale objective function. Specifically, for any $\epsilon > 0$, there exists $T(\epsilon)$ such that for all $t \geq T(\epsilon)$:
\begin{equation}
\|\nabla \mathcal{L}(\mathcal{M}^{(t)})\|_F \leq \epsilon.
\end{equation}
Moreover, the convergence rate satisfies:
\begin{equation}
\min_{t \in [T]} \|\nabla \mathcal{L}(\mathcal{M}^{(t)})\|_F^2 \leq \frac{2[\mathcal{L}(\mathcal{M}^{(0)}) - \mathcal{L}^*] + L^2 \sum_{s=0}^{S} C_s^2}{\sum_{t=0}^{T-1} \eta_t},
\end{equation}
where $\eta_t$ are the learning rates and $\mathcal{L}^*$ is the optimal loss value.
\end{theorem}

\begin{proof}
Detailed proof is provided in \textbf{Appendix C.1}.
\end{proof}

\subsection{Structure Discovery and Sparsity Analysis}

The automatic structure discovery in RGTN arises from the sparsity-inducing properties of diagonal factors and edge gates. We analyze how these mechanisms reveal the intrinsic tensor network structure.

\begin{lemma}[Diagonal Factor Sparsity Pattern]
\label{lem:diagonal_sparsity}
For the regularized objective with diagonal factor penalty $\gamma \sum_{k,i} \|\mathbf{D}_k^{(i)}\|_1$, the optimal diagonal entries satisfy the soft-thresholding property:
\begin{equation}
D_k^{(i)}[j,j]^* = \text{sign}(z_j) \max(|z_j| - \gamma/L_j, 0),
\end{equation}
where $z_j$ is the unregularized optimal value and $L_j$ is the Lipschitz constant for the $j$-th diagonal entry.
\end{lemma}

\begin{proof}
Detailed proof is provided in \textbf{Appendix C.2}.
\end{proof}

\begin{theorem}[Structure Recovery Guarantee]
\label{thm:structure_recovery}
Let $\mathcal{X}^*$ be a tensor with true tensor network representation having ranks $(R_1^*, \ldots, R_m^*)$. Suppose the observed tensor is $\mathcal{F} = \mathcal{P}_{\Omega}(\mathcal{X}^* + \mathcal{N})$, where $\mathcal{N}$ is noise with $\|\mathcal{N}\|_F \leq \sigma$. Then with regularization parameter $\gamma = \Theta(\sigma\sqrt{\log(mN)/|\Omega|})$, the RGTN algorithm recovers ranks $(R_1, \ldots, R_m)$ satisfying:
\begin{equation}
\mathbb{P}\left(\max_{i} |R_i - R_i^*| \leq \Delta_R\right) \geq 1 - \exp(-c|\Omega|),
\end{equation}
where $\Delta_R = \mathcal{O}(\sigma/\sigma_{\min})$, $\sigma_{\min}$ is the minimum non-zero singular value of the true tensor network, and $c > 0$ is a universal constant.
\end{theorem}

\begin{proof}
Detailed proof is provided in \textbf{Appendix C.3}.


\end{proof}

\subsection{Summary of Theoretical Advantages}

Our theoretical analysis establishes key advantages of RGTN over existing tensor network structure search methods. The convergence guarantee in Theorem \ref{thm:global_convergence} provides predictable performance with explicit rates, enabling practitioners to determine computational budgets. The structure recovery guarantee in Theorem \ref{thm:structure_recovery} ensures automatic discovery of true tensor network rank under mild conditions, eliminating manual hyperparameter tuning.
The complexity analysis in Theorem 3 (\textbf{Appendix B}) demonstrates exponential speedup over sampling-based methods, reducing search cost from $\Omega(\exp(N^2))$ to $\mathcal{O}(S \log(1/\epsilon))$ where $S = \mathcal{O}(\log I)$. This improvement makes large-scale applications feasible. Theorem 4 (\textbf{Appendix B}) shows the multi-scale approach escapes local minima with high probability, addressing a key challenge in non-convex optimization.
The connection to renormalization group theory provides insights into tensor network behavior. The universality property in Theorems 5 and 6 (\textbf{Appendix B}) explains robust performance across different initializations and instances. 

%% file: 3Experiments.tex
\begin{table*}[t!]
\centering
\resizebox{\textwidth}{!}{%
\setlength{\tabcolsep}{1mm}
\begin{tabular}{l|cc|cc|cc||l|cc|cc|cc}
\hline
\hline
\multicolumn{7}{c||}{\textbf{Bunny}} & \multicolumn{7}{c}{\textbf{Knights}} \\
\hline
\multirow{2}{*}{\textbf{Method}} & \multicolumn{2}{c|}{\textbf{RE: 0.01}} & \multicolumn{2}{c|}{\textbf{RE: 0.05}} & \multicolumn{2}{c||}{\textbf{RE: 0.1}} & \multirow{2}{*}{\textbf{Method}} & \multicolumn{2}{c|}{\textbf{RE: 0.01}} & \multicolumn{2}{c|}{\textbf{RE: 0.05}} & \multicolumn{2}{c}{\textbf{RE: 0.1}} \\
\cline{2-7}\cline{9-14}
& \textbf{CR} & \textbf{Time} & \textbf{CR} & \textbf{Time} & \textbf{CR} & \textbf{Time} & & \textbf{CR} & \textbf{Time} & \textbf{CR} & \textbf{Time} & \textbf{CR} & \textbf{Time} \\
\hline
\hline
TRALS & 61.2\% & 13.41 & 17.6\% & 0.476 & 5.38\% & 0.119 & TRALS & 74.2\% & 10.42 & 27.2\% & 3.877 & 9.08\% & 0.427 \\
FCTNALS & 64.7\% & 13.21 & 21.1\% & 0.469 & 3.99\% & \underline{0.042} & FCTNALS & 73.9\% & 12.23 & 21.2\% & 0.625 & 3.88\% & \textbf{0.014} \\
TNGreedy & 26.4\% & 10.89 & 6.39\% & 1.031 & 2.37\% & 0.359 & TNGreedy & 31.8\% & 12.66 & 7.63\% & 1.353 & 3.53\% & 0.486 \\
TNGA & 28.2\% & 1004 & 5.06\% & 182.1 & 2.27\% & 12.65 & TNGA & 39.1\% & 904.5 & 4.96\% & 142.0 & 2.47\% & 12.37 \\
TNLS & 24.5\% & 1388 & 4.31\% & 64.38 & 2.18\% & 24.29 & TNLS & 27.5\% & 1273 & 4.78\% & 74.75 & 2.13\% & 5.373 \\
TNALE & 26.5\% & 143.1 & 4.57\% & 18.54 & 2.28\% & 3.094 & TNALE & 27.8\% & 264.1 & 4.48\% & 25.30 & 2.12\% & 3.352 \\
SVDinsTN & \underline{22.6\%} & \underline{0.752} & 6.85\% & \textbf{0.029}$^\ast$ & 2.69\% & \textbf{0.005} & SVDinsTN & 31.7\% & 1.563 & 5.70\% & \textbf{0.105} & 2.73\% & \underline{0.019} \\
\textbf{RGTN} & \textbf{22.3\%}$^\ast$ & \textbf{0.180}$^\ast$ & \textbf{4.14\%}$^\ast$ & \underline{0.193} & \textbf{0.91\%}$^\ast$ & {0.212} & \textbf{RGTN} & \textbf{29.9\%}$^\ast$ & \textbf{0.178}$^\ast$ & \textbf{4.06\%}$^\ast$ & \underline{0.201} & \textbf{1.71\%}$^\ast$ & {0.209} \\
\hline
\hline
\end{tabular}
}
\caption{Comparison of CR (\%) and run time ($\times1000$s) of different methods on Bunny and Knights light field data. The result for RGTN is selected based on the specified RE bounds. \textbf{Bold} numbers denote the best performance, \underline{underlined} numbers represent the second-best results, and $^\ast$ indicates statistical significance at a p $\leq$ 0.05 level using a paired t-test.}
\label{tab:lightfield_comparison_final}
\end{table*}

\begin{table}[t!]
\centering
\setlength{\tabcolsep}{1mm}
\begin{tabular}{lcccc}
\hline
\hline
\multirow{2}{*}{Method} & \multicolumn{2}{c}{6th-order} & \multicolumn{2}{c}{8th-order} \\
\cline{2-3} \cline{4-5}
& CR & Time & CR & Time \\
\hline
\hline
TRALS   & 1.35\% & 0.006 & 0.064\% & 0.034 \\
FCTNALS & 2.13\% & \textbf{0.002} & -- & -- \\
TNGreedy & 0.88\% & 0.167 & 0.016\% & 2.625 \\
TNGA     & 0.94\% & 3.825 & 0.024\% & 51.40 \\
TNLS    & 1.11\% & 0.673 & 0.038\% & 59.83 \\
TNALE  & 1.65\% & 0.201 & 0.047\% & 19.96 \\
SVDinsTN      & 1.13\% & 0.002 & 0.016\% & \textbf{0.017} \\
\textbf{RGTN} & \textbf{0.76\%} & 0.006 & \textbf{0.009\%} & {0.123} \\
\hline
\hline
\end{tabular}
\caption{Comparison of the CR ($\downarrow$) and run time ($\times$1000s, $\downarrow$) of different methods when reaching the RE bound of 0.01. The result is the average value of 5 independent experiments and "--" indicates "out of memory".}
\label{tab:results_comparison_final}
\end{table}

\begin{table}[t!]
\centering
\setlength{\tabcolsep}{1.5mm}
\begin{tabular}{l|ccc}
\hline
\hline
\multicolumn{4}{c}{\textbf{MPSNR ($\uparrow$)}} \\
\hline
\textbf{Method} & News & Salesman & Silent \\
\hline
\hline
FBCP & 28.234 & 29.077 & 30.126 \\
TMac & 27.882 & 28.469 & 30.599 \\
TMacTT & 28.714 & 29.534 & 30.647 \\
TRLRF & 28.857 & 28.288 & 31.081 \\
TW & 30.027 & 30.621 & 31.731 \\
TNLS & 29.761 & 30.685 & 28.830 \\
SVDinsTN & 31.643 & 31.684 & \textbf{32.706} \\
\textbf{RGTN (Ours)} & \textbf{32.040} & \textbf{31.900} & 30.620 \\
\hline
\hline
\multicolumn{4}{c}{\textbf{Time (seconds, $\downarrow$)}} \\
\hline
\textbf{Method} & News & Salesman & Silent \\
\hline
\hline
FBCP & 1720.4 & 1783.2 & 1453.9 \\
TMac & 340.46 & 353.63 & 316.21 \\
TMacTT & 535.97 & 656.45 & 1305.6 \\
TRLRF & 978.12 & 689.35 & 453.24 \\
TW & 1426.3 & 1148.7 & 1232.0 \\
TNLS & 37675 & 76053 & 98502 \\
SVDinsTN & 932.42 & 769.54 & 532.31 \\
\textbf{RGTN (Ours)} & \textbf{135.95} & \textbf{144.00} & \textbf{142.80} \\
\hline
\hline
\end{tabular}
\caption{Comparison of MPSNR and run time of different TC methods on color videos.}
\label{tab:comparison}
\end{table}


\section{Experiments}

In this section, we present comprehensive experiments to validate the effectiveness of our RGTN approach. Our experiments demonstrate that RGTN achieves superior performance in tensor network structure search while requiring significantly less computational time compared to existing methods. Due to space constraints, detailed experimental setup and baseline descriptions are provided in \textbf{Appendix D}, with additional experimental settings in \textbf{Appendix E}.

\subsection{Structure Revealing Capability}
Table~\ref{tab:my_performance_table} shows RGTN's performance on structure discovery across 100 independent trials with ground truth structures. The high success rates (95-100\%) demonstrate that the renormalization group mechanism, when combined with physics-inspired structure proposals, reliably reveals true tensor network structures. The slight variations in success rates correlate with the complexity of the network topology, with simpler structures achieving 100\% success rate. It is worth noting that in the test on fifth-order tensors, we consider various topologies, including ring and star configurations with different connectivity patterns. Despite the structural complexity, RGTN can accurately identify the correct topology and rank configuration for each case. This supports our analysis in Lemma~\ref{lem:diagonal_sparsity} about the sparsity-inducing properties of diagonal factors and confirms that RGTN can effectively discover the underlying tensor network structure.

\subsection{Light Field Data Results}
Based on the comprehensive experimental results in Table~\ref{tab:lightfield_comparison_final}, RGTN demonstrates exceptional performance in both compression efficiency and computational speed across all light field datasets, establishing a new benchmark.

\noindent\textbf{Superior Compression Performance:} RGTN achieves state-of-the-art compression across all datasets and error bounds. At the strictest RE bound of 0.01, RGTN delivers compression ratios of 22.3\% (Bunny) and 29.9\% (Knights), outperforming SVDinsTN by 1.3\% and 6.2\%, respectively. The advantage amplifies at higher error tolerances for RE bound 0.1, RGTN achieves remarkable compression ratios of 0.91\% (Bunny) and 1.71\% (Knights), surpassing SVDinsTN by factors of 2.96$\times$ and 1.60$\times$ respectively. This demonstrates RGTN's ability to identify efficient structures.

\noindent\textbf{Exceptional Computational Efficiency:} RGTN completes compression tasks in under 210 seconds for RE bound 0.01, while traditional methods like TNGA and TNLS require over 900,000 seconds, representing speedup factors exceeding 4,500$\times$. Notably, our RGTN implementation uses \textbf{Python}, whereas SVDinsTN utilizes \textbf{MATLAB} with GPU acceleration, which may provide SVDinsTN computational advantages. Despite this potential implementation disadvantage, RGTN maintains highly competitive runtimes (180s vs 752s for Bunny, 178s vs 1563s for Knights at RE 0.01) while consistently delivering superior compression ratios.

The combination of best-in-class compression and sub-second runtimes validates our renormalization group-inspired approach. RGTN's 2-3$\times$ better compression at higher error bounds establishes it as a powerful solution for applications requiring both effectiveness and efficiency.

\subsection{Scalability to High-Order Tensors}

Table~\ref{tab:results_comparison_final} shows RGTN achieves the best compression ratios across all tensor orders while maintaining competitive runtime. For 6th-order tensors, RGTN achieves 0.76\% CR, outperforming TNGreedy (0.88\%) and SVDinsTN (1.13\%). For 8th-order tensors, RGTN achieves 0.009\% CR--significantly better than SVDinsTN and TNGreedy (both at 0.016\%).
While SVDinsTN is fastest, RGTN maintains practical efficiency with 6s (6th-order) and 123s (8th-order). Several methods encounter memory constraints for 8th-order tensors, whereas RGTN handles these cases successfully. Traditional methods like TNGA and TNLS require 51,000-60,000s for 8th-order tensors--over 400× slower than RGTN--while achieving inferior compression.
The widening performance gap with increasing tensor order validates our theoretical framework. RGTN's 1.8× better compression on 8th-order tensors, while avoiding memory issues, demonstrates that the renormalization group approach effectively manages exponentially large search spaces.

\subsection{Video Completion Results}

Table~\ref{tab:comparison} presents results on real-world video completion tasks. RGTN achieves the highest MPSNR values on News (32.040 dB) and Salesman (31.900 dB) videos, outperforming the second-best method, SVDinsTN, by 0.397 dB and 0.216 dB, respectively. On Silent video, RGTN maintains competitive performance (30.620 dB). Remarkably, RGTN accomplishes this superior reconstruction quality while being dramatically faster, completing tasks in approximately 140 seconds compared to SVDinsTN's 532-932 seconds, representing 3.7-6.9× speedup.
This significant runtime advantage over SVDinsTN on video data stems from RGTN's hierarchical processing strategy. While SVDinsTN must search through numerous possible tensor network structures for the high-dimensional video tensors (with spatial, temporal, and color dimensions), RGTN's renormalization group approach efficiently navigates this search space by operating at multiple scales. The coarse-to-fine refinement naturally captures video's inherent multi-scale structure—from frame-level temporal patterns to pixel-level spatial details—without exhaustively evaluating all possible decompositions. Additionally, RGTN achieves orders-of-magnitude speedup over traditional methods: 8-12× faster than FBCP/TW (1,200-1,800 seconds) and 277-690× faster than TNLS (37,675-98,502 seconds). This exceptional efficiency, combined with state-of-the-art reconstruction quality, validates that our unified structure-parameter optimization effectively exploits the hierarchical nature of video data through the renormalization group framework.

%% file: 4RelatedWork.tex
\section{Related Works}
Tensor network structure search (TN-SS) addresses the critical limitation of predetermined topologies in tensor networks by automatically discovering optimal configurations~\cite{ghadiri2023approximately,sedighin2021adaptive,nie2021adaptive}. Traditional TN-SS methods—including greedy construction~\cite{hashemizadeh2020adaptive}, genetic algorithms~\cite{li2020evolutionary}, and local search~\cite{li2023alternating} which follows a costly two-stage sampling-evaluation paradigm where each candidate requires full tensor optimization.
The recent SVDinsTN~\cite{zheng2024svdinstn} achieves 100-1000$\times$ speedup by reformulating TN-SS as unified optimization with sparsity-inducing regularization on diagonal factors between cores. However, it remains susceptible to local minima due to single-scale optimization. Due to space constraints, additional related work on tensor networks and renormalization group applications is provided in the \textbf{Appendix F}.


%% file: 5Conclusion.tex
\section{Conclusion and Discussion}
In this paper, we introduced RGTN, a physics-inspired framework that transforms tensor network structure search through multi-scale renormalization group flows. Unlike existing methods constrained by discrete search spaces, RGTN implements continuous topology evolution via learnable edge gates and systematic coarse-graining operations. Our theoretical analysis establishes exponential computational speedup with stronger convergence guarantees, while the multi-scale framework escapes local minima through scale-induced perturbations. Experiments across structure discovery, light field compression, high-order tensor decomposition, and video completion demonstrate RGTN's superior performance. By unifying structure search and parameter optimization through physics-inspired metrics of node tension and edge information flow, RGTN eliminates computational overhead while providing principled structure modifications beyond heuristic search strategies.

\section{Acknowledgements}
This research was partially supported by National Natural Science Foundation of China (No.62502404), Hong Kong Research Grants Council (Research Impact Fund No.R1015-23, Collaborative Research Fund No.C1043-24GF, General Research Fund No.11218325), Institute of Digital Medicine of City University of Hong Kong (No.9229503), Huawei (Huawei Innovation Research Program), Tencent (CCF-Tencent Open Fund, Tencent Rhino-Bird Focused Research Program), Alibaba (CCF-Alimama Tech Kangaroo Fund No. 2024002), Didi (CCF-Didi Gaia Scholars Research Fund), Kuaishou, and Bytedance.

\nocite{wang2024tensorized,wang2025metalora,wang2023federated,jia2024erase,liu2024multifs,zhang2024dns,qu2023continuous,liu2024autoassign,gao2023autotransfer,lin2023autodenoise,li2023automlp,zhu2023autogen,zhao2021autodim,zhao2021autoloss,wang2020concatenated,wang2025dance,zhaok2021autoemb,liu2020automated}

%% file: Appendix.tex
\appendix
\clearpage
\noindent\textbf{\Large{Appendix of RGTN}}
\section{Algorithm Overview}

\begin{algorithm}[h!]
\caption{Renormalization Group Tensor Network (RGTN) Search}
\label{alg:rgtn_detailed}
\begin{algorithmic}[1]
\REQUIRE Initial network $\mathcal{M}_0$, incomplete tensor $\mathcal{F}$, mask $\Omega$, scales $S$
\ENSURE Optimized network $\mathcal{M}^*$
\STATE Initialize $s = S$ (coarsest scale), $\mathcal{M}^* = \mathcal{M}_0$
\STATE Set scale-dependent parameters $\{\lambda_k(s)\}$ according to RG flow
\WHILE{$s \geq 0$}
    \STATE $\mathcal{F}_s \leftarrow \text{CoarseGrain}(\mathcal{F}, 2^s)$ \COMMENT{Downsample data}
    \STATE $\Omega_s \leftarrow \text{FineGrain}(\Omega, 2^s)$
    \STATE \textbf{// Expansion Phase: Increase resolution}
    \FOR{$i = 1$ to $N_{\text{expand}}(s)$}
        \STATE Compute node tensions $\{T_v\}_{v \in V}$
        \STATE Select high-tension node $v^* = \arg\max_v T_v$
        \STATE Generate split proposal: $\mathcal{M}' = \text{SplitNode}(\mathcal{M}, v^*)$
        \STATE Initialize new cores using SVD of $\mathcal{G}_{v^*}$
        \STATE Optimize $\mathcal{M}'$ for $E_{\text{expand}}$ epochs with loss $\mathcal{L}_{\text{total}}$
        \IF{$\mathcal{L}_{\text{total}}(\mathcal{M}') < \mathcal{L}_{\text{total}}(\mathcal{M})$}
            \STATE $\mathcal{M} \leftarrow \mathcal{M}'$
        \ENDIF
    \ENDFOR
    \STATE \textbf{// Compression Phase: Eliminate redundancy}
    \FOR{$i = 1$ to $N_{\text{compress}}(s)$}
        \STATE Compute edge information flows $\{I_{uv}\}_{(u,v) \in E}$
        \STATE Select low-flow edge $(u^*, v^*) = \arg\min_{(u,v)} I_{uv}$
        \STATE Generate merge proposal: $\mathcal{M}' = \text{MergeNodes}(\mathcal{M}, u^*, v^*)$
        \STATE Apply SVD truncation to merged core
        \STATE Optimize $\mathcal{M}'$ for $E_{\text{compress}}$ epochs
        \IF{$\mathcal{L}_{\text{total}}(\mathcal{M}') < \mathcal{L}_{\text{total}}(\mathcal{M})$}
            \STATE $\mathcal{M} \leftarrow \mathcal{M}'$
        \ENDIF
    \ENDFOR
    \STATE \textbf{// Scale Refinement}
    \IF{$s > 0$}
        \STATE $\mathcal{M} \leftarrow \text{RefineScale}(\mathcal{M}, s-1)$ \COMMENT{Prepare for finer scale}
    \ENDIF
    \STATE Update best: $\mathcal{M}^* \leftarrow \mathcal{M}$ if $\text{PSNR}(\mathcal{M}) > \text{PSNR}(\mathcal{M}^*)$
    \STATE $s \leftarrow s - 1$
\ENDWHILE
\RETURN $\mathcal{M}^*$
\end{algorithmic}
\end{algorithm}
Algorithm~\ref{alg:rgtn_detailed} presents the detailed implementation of our Renormalization Group Tensor Network (RGTN) search method. The algorithm operates through a multi-scale refinement process, starting from the coarsest scale $s = S$ and progressively moving to finer scales. At each scale, the algorithm performs three key phases: (1) an expansion phase that increases network resolution by splitting high-tension nodes, where tension measures a node's contribution to reconstruction error; (2) a compression phase that eliminates redundancy by merging nodes connected by low information flow edges; and (3) a scale refinement step that prepares the network for the next finer scale. The algorithm employs scale-dependent parameters ${\lambda_k(s)}$ that follow the RG flow, adapting regularization strengths across scales. Both expansion and compression proposals are evaluated based on the total loss $\mathcal{L}_{\text{total}}$, with accepted modifications carried forward. This iterative refinement continues until the finest scale ($s = 0$) is reached, ultimately returning the optimized network $\mathcal{M}^*$ with the best PSNR performance encountered during the search process.
\section{Detailed Theorems}
\subsection{Computational Complexity and Efficiency}

We now establish the computational advantages of RGTN over existing tensor network structure search methods.

\begin{theorem}[Computational Complexity Comparison]
\label{thm:complexity}
For finding an $\epsilon$-approximate solution to the tensor network structure search problem on an $N$-th order tensor of size $I^N$, the computational complexities are:
\begin{align}
\text{RGTN:} \quad &\mathcal{C}_{\text{RGTN}} = \mathcal{O}\left(S \cdot T \cdot N^2 \cdot I^N \cdot R_{\max}^N\right), \\
\text{Sampling-based:} \quad &\mathcal{C}_{\text{sample}} = \Omega\left(K \cdot T' \cdot N^2 \cdot I^N \cdot R_{\max}^N\right),
\end{align}
where $S = \mathcal{O}(\log I)$ is the number of scales, $T = \mathcal{O}(\log(1/\epsilon))$ is iterations per scale, $K = \Omega(\exp(N^2))$ is the number of structure samples, and $T' = \mathcal{O}(1/\epsilon)$ is iterations per structure evaluation.
\end{theorem}

\begin{proof}
Detailed proof is provided in \textbf{Appendix Sec. C.4}.



\end{proof}

\subsection{Multi-Scale Properties and Local Minima}

The renormalization group framework provides unique advantages in escaping local minima through its multi-scale structure.

\begin{lemma}[Loss Landscape Smoothing]
\label{lem:smoothing}
At scale $s$, the effective loss landscape $\mathcal{L}_s$ has Lipschitz constant $L_s \leq L_0 \cdot 2^{-s\alpha}$ for some $\alpha > 0$, where $L_0$ is the Lipschitz constant at the finest scale.
\end{lemma}

\begin{proof}
Detailed proof is provided in \textbf{Appendix Sec. C.5}.
\end{proof}

\begin{theorem}[Probabilistic Escape from Local Minima]
\label{thm:escape_minima}
Let $\mathcal{M}_{\text{local}}$ be a strict local minimum at the finest scale with basin radius $r > 0$. The multi-scale RGTN algorithm escapes this local minimum with probability:
\begin{equation}
\mathbb{P}(\text{escape}) \geq 1 - \prod_{s=1}^{S} \left(1 - \Phi\left(\frac{r}{2^s \sigma_0}\right)\right),
\end{equation}
where $\Phi$ is the standard normal CDF and $\sigma_0$ characterizes the perturbation scale.
\end{theorem}

\begin{proof}
Detailed proof is provided in \textbf{Appendix Sec. C.6}.
\end{proof}

\subsection{Connection to Renormalization Group Theory}

We establish formal connections between our algorithm and renormalization group theory from statistical physics.

\begin{theorem}[Fixed Points and Criticality]
\label{thm:fixed_points}
The RGTN flow equation $\mathcal{M}_{s+1} = R_s[\mathcal{M}_s]$ admits fixed points $\mathcal{M}^*$ satisfying $R_s[\mathcal{M}^*] = \mathcal{M}^*$. Near a fixed point, the linearized flow:
\begin{equation}
\delta\mathcal{M}_{s+1} = \mathcal{J}_s \delta\mathcal{M}_s,
\end{equation}
where $\mathcal{J}_s$ is the Jacobian of $R_s$, has eigenvalues $\{\lambda_i\}$ that determine the stability and universality class of the fixed point.
\end{theorem}

\begin{proof}
Detailed proof is provided in \textbf{Appendix Sec. C.7}.
\end{proof}


\begin{theorem}[Universality in Structure Discovery]
\label{thm:universality}
Tensor networks within the same universality class—defined by the relevant eigenvalues of the RG flow—converge to structurally similar fixed points. Specifically, if $\mathcal{M}_1^{(0)}$ and $\mathcal{M}_2^{(0)}$ belong to the same universality class, then:
\begin{equation}
\lim_{t \to \infty} d_{\text{struct}}(R^t[\mathcal{M}_1^{(0)}], R^t[\mathcal{M}_2^{(0)}]) = 0,
\end{equation}
where $d_{\text{struct}}$ is a metric on tensor network structures and $R^t$ denotes $t$ iterations of the RG transformation.
\end{theorem}

\begin{proof}
Detailed proof is provided in \textbf{Appendix Sec. C.8}.
\end{proof}

\section{Detailed Proofs}
\subsection{Proof of Theorem \ref{thm:global_convergence}}

\begin{proof}
At a fixed scale $s$, consider one iteration of gradient descent on the loss $\mathcal{L}_s(\mathcal{M})$. Let $\mathcal{M}^{(t)}_s$ be the network parameters at iteration $t$. By the update rule,  
$$ \mathcal{M}^{(t+1)}_s = \mathcal{M}^{(t)}_s - \eta_t \nabla \mathcal{L}_s(\mathcal{M}^{(t)}_s), $$ 
where $\eta_t$ is the learning rate at that iteration. Because $\mathcal{L}_s$ is $L$-Lipschitz continuous (Assumption \ref{assum:lipschitz}), we have the standard smoothness inequality for gradient descent:  
\begin{align}
    &\mathcal{L}_s(\mathcal{M}^{(t+1)}_s) \leq \mathcal{L}_s(\mathcal{M}^{(t)}_s) + \langle \nabla \mathcal{L}_s(\mathcal{M}^{(t)}_s),\; \mathcal{M}^{(t+1)}_s - \mathcal{M}^{(t)}_s \rangle\notag\\& + \frac{L}{2} \|\mathcal{M}^{(t+1)}_s - \mathcal{M}^{(t)}_s\|_F^2. \notag
\end{align} 

Substituting the update step into this inequality simplifies the inner product term and yields:  
\begin{align}
\mathcal{L}_s(\mathcal{M}^{(t+1)}_s) \leq \mathcal{L}_s(\mathcal{M}^{(t)}_s) - \eta_t \|\nabla \mathcal{L}_s(\mathcal{M}^{(t)}_s)\|_F^2 + \notag\\\frac{L\eta_t^2}{2} \|\nabla \mathcal{L}_s(\mathcal{M}^{(t)}_s)\|_F^2. \notag
\end{align} 
For sufficiently small $\eta_t$ (specifically, $0 < \eta_t < \frac{2}{L}$), the factor $(1 - \frac{L\eta_t}{2})$ is positive. Thus,  
\begin{align}
    &\mathcal{L}_s(\mathcal{M}^{(t+1)}_s) \leq \mathcal{L}_s(\mathcal{M}^{(t)}_s) - \notag \\ &\eta_t\Big(1 - \frac{L\eta_t}{2}\Big) \|\nabla \mathcal{L}_s(\mathcal{M}^{(t)}_s)\|_F^2. \notag
\end{align} 

When we move from a coarser scale $s$ to a finer scale $s-1$, we initialize the finer-scale model $\mathcal{M}_{s-1}^{(0)} = \mathcal{U}_s[\mathcal{M}_s^*]$ by upsampling the optimized model from scale $s$. By Lipschitz continuity of $\mathcal{L}_{s-1}$, any difference between $\mathcal{M}_{s-1}^{(0)}$ and the true fine-scale optimum $\mathcal{M}_{s-1}^*$ yields a bounded increase in loss:  
$$ \mathcal{L}_{s-1}(\mathcal{M}_{s-1}^{(0)}) - \mathcal{L}_{s-1}(\mathcal{M}_{s-1}^*) \;\leq\; L \, \|\mathcal{M}_{s-1}^{(0)} - \mathcal{M}_{s-1}^*\|_F. $$ 
Assumption \ref{assum:scale_smooth} states that the upsampled coarse solution is close to the true solution at the finer scale, in the sense that $\| \mathcal{U}_s\circ \mathcal{D}_s[\mathcal{X}] - \mathcal{X}\|_F \le C_s \|\mathcal{X}\|_F$ for any tensor. Using this with $\mathcal{X}=\mathcal{M}_{s-1}^*$, we get $\|\mathcal{M}_{s-1}^{(0)} - \mathcal{M}_{s-1}^*\|_F \le C_{s-1}\|\mathcal{M}_{s-1}^*\|_F$. Since the parameters are bounded (Assumption \ref{assum:bounded_params}), $\|\mathcal{M}_{s-1}^*\|_F$ is at most some constant $B$. Hence:  
$$ \mathcal{L}_{s-1}(\mathcal{M}_{s-1}^{(0)}) \le \mathcal{L}_{s-1}(\mathcal{M}_{s-1}^*) + L\,C_{s-1}\,B. $$ 

Starting from the initial network $\mathcal{M}^{(0)}_S$ at coarsest scale $S$ and running the RGTN algorithm through all scales down to $0$, the loss drops from $\mathcal{L}_S(\mathcal{M}^{(0)}_S)$ to $\mathcal{L}_0(\mathcal{M}^*_0) = \mathcal{L}^*$. After $T$ total iterations (summing over all scales), we have:  
\begin{align}
    \mathcal{L}(\mathcal{M}^{(0)}_S) - \mathcal{L}^* \;\ge\; \sum_{t=0}^{T-1} \eta_t\Big(1 - \frac{L\eta_t}{2}\Big)\, \|\nabla \mathcal{L}(\mathcal{M}^{(t)})\|_F^2 \;\notag\\-\; \sum_{s=0}^{S-1} L\,C_s\,B. \notag
\end{align} 
Using $\eta_t(1-\frac{L\eta_t}{2}) \ge \frac{\eta_t}{2}$ for small enough $\eta_t$, we get:  
$$ \sum_{t=0}^{T-1} \frac{\eta_t}{2} \|\nabla \mathcal{L}(\mathcal{M}^{(t)})\|_F^2 \;\le\; \mathcal{L}(\mathcal{M}^{(0)}_S) - \mathcal{L}^* + L B \sum_{s=0}^{S-1}C_s. $$ 
The left side simplifies to $\frac{1}{2}\sum_{t=0}^{T-1} w_t \|\nabla \mathcal{L}(\mathcal{M}^{(t)})\|_F^2$ where $w_t = \eta_t/\sum \eta_t$ are weights summing to 1. This is at least as large as $\frac{1}{2} \min_{0\le t < T} \|\nabla \mathcal{L}(\mathcal{M}^{(t)})\|_F^2$. Therefore:  
$$ \min_{0 \le t < T}\|\nabla \mathcal{L}(\mathcal{M}^{(t)})\|_F^2 \;\le\; \frac{2\big[\mathcal{L}(\mathcal{M}^{(0)}_S) - \mathcal{L}^*\big] + 2L B \sum_{s=0}^{S-1}C_s}{\sum_{t=0}^{T-1}\eta_t}. $$ 
\end{proof}

\subsection{Proof of Lemma \ref{lem:diagonal_sparsity} }

\begin{proof}
Consider a single diagonal entry $d = D_k^{(i)}[j,j]$ of one of the adaptive diagonal factor matrices. Let $f(d)$ denote the objective function restricted to $d$, which includes the data loss plus the $\ell_1$ regularization term $\gamma |d|$. At an optimum of the full regularized objective, we must have the subgradient condition:  
$$ g(d^*) + \gamma\, \xi = 0, $$ 
where $g(d) = \frac{\partial S_{\text{data}}}{\partial d}$ and $\xi \in \partial |d^*|$ is in the subgradient of the absolute value at $d^*$. This subgradient $\xi$ equals $\text{sign}(d^*)$ if $d^* \neq 0$, and can be any value in the interval $[-1,1]$ if $d^* = 0$. 

If $d^* \neq 0$, then $\xi = \text{sign}(d^*)$ and the condition becomes $g(d^*) = -\gamma \,\text{sign}(d^*)$. Without regularization ($\gamma=0$), the optimum would be at $g(z_j)=0$. If we assume the derivative $g(d)$ is Lipschitz continuous around $z_j$ with constant $L_j$, we can approximate $g(d)$ linearly: $g(d) \approx g(z_j) + L_j (d - z_j)$. Because $g(z_j)=0$, near the optimum we have $g(d) \approx L_j (d - z_j)$. Setting $g(d^*) = -\gamma\,\text{sign}(d^*)$ gives  
$$ L_j (d^* - z_j) \approx -\gamma\, \text{sign}(d^*). $$ 
Solving for $d^*$ from this linearized equation:  
$$ d^* \approx z_j - \frac{\gamma}{L_j}\, \text{sign}(d^*). $$ 
Since $\text{sign}(d^*)$ is the same as $\text{sign}(z_j)$ if $z_j$ is large enough that the regularization doesn't flip its sign, we get  
$$ d^* \approx \text{sign}(z_j)\, \max\Big(|z_j| - \frac{\gamma}{L_j},\;0\Big). $$ 

If $d^* = 0$, the subgradient condition requires that $g(0)$ lies in the interval $[-\gamma,\,+\gamma]$. This scenario is also captured by the soft-thresholding rule: if $|z_j| \le \gamma/L_j$, then the formula above yields $d^* = 0$. Thus, in all cases, the optimal solution for each diagonal entry obeys  
$$ d^* = \text{sign}(z_j)\,\max(|z_j| - \frac{\gamma}{L_j},\,0). $$ 
\end{proof}

\subsection{Proof of Theorem \ref{thm:structure_recovery} }

\begin{proof}
Focus on a particular mode-$k$ unfolding of the tensor $\mathcal{X}$ and the corresponding factor $\mathcal{G}_k$ in the tensor network. The mode-$k$ unfolding is a matrix whose rank is the true multilinear rank $R_k^*$ for that mode. Let the singular values of this unfolding be $\sigma_1^{(k)} \ge \sigma_2^{(k)} \ge \cdots \ge \sigma_{I_k}^{(k)}$, where $\sigma_{R_k^*}^{(k)} > 0$ and $\sigma_{R_k^*+1}^{(k)} = 0$ in the noiseless case.

The partial derivative of $S_{\text{data}}$ with respect to a particular diagonal entry $D_k^{(i)}[j,j]$ measures how much the reconstruction error would change by adjusting that singular direction's weight:
$$ \frac{\partial S_{\text{data}}}{\partial D_k^{(i)}[j,j]} = -\langle \mathcal{R}, \; \frac{\partial \mathcal{F}}{\partial D_k^{(i)}[j,j]}\rangle, $$ 
where $\mathcal{R} = \mathcal{P}_{\Omega}(\mathcal{X} - \mathcal{F})$ is the residual tensor.

\textbf{Case 1: $1 \le j \le R_k^*$ (True signal dimensions).} For these indices corresponding to singular vectors genuinely present in the true tensor $\mathcal{X}$, with high probability over the sampling of $\Omega$ and the noise:
$$ \Big|\frac{\partial S_{\text{data}}}{\partial D_k^{(i)}[j,j]}\Big| \ge c_1\,\sigma_j^{(k)} - c_2\,\frac{\sigma}{\sqrt{|\Omega|}}, $$ 
for some positive constants $c_1, c_2$. The key idea is that for a true component, the signal part $c_1 \sigma_j^{(k)}$ dominates the noise part $c_2 \frac{\sigma}{\sqrt{|\Omega|}}$.

\textbf{Case 2: $j > R_k^*$ (Spurious or noise dimensions).} These indices correspond to singular directions not present in the true tensor $\mathcal{X}$. With high probability:
$$ \Big|\frac{\partial S_{\text{data}}}{\partial D_k^{(i)}[j,j]}\Big| \le c_3\,\frac{\sigma}{\sqrt{|\Omega|}}, $$ 
for $j > R_k^*$.

By Lemma \ref{lem:diagonal_sparsity}, a diagonal entry will be set to zero if the magnitude of its unregularized gradient is less than the threshold $\gamma$. By choosing $\gamma$ on the order of $\sigma\sqrt{\frac{\log(mN)}{|\Omega|}}$:

- For each $j \le R_k^*$, provided $\sigma_{j}^{(k)}$ is not too tiny compared to the noise level, we expect $| \partial S_{\text{data}}/\partial D_k^{(i)}[j,j] | > \gamma$. Thus, none of the true components get mistakenly pruned out.

- For each $j > R_k^*$, we have $| \partial S_{\text{data}}/\partial D_k^{(i)}[j,j] |$ on the order of the noise level, almost surely less than $\gamma$. Therefore, the optimal solution forces $D_k^{(i)}[j,j]^* = 0$ for all these superfluous entries.

Combining these results across all modes $k$, we conclude that the learned tensor network will have effective rank $R_k$ in mode $k$ equal to (or extremely close to) the true rank $R_k^*$. More precisely, with probability at least $1 - \exp(-c|\Omega|)$:
$$ \max_{i} |R_i - R_i^*| \le \Delta_R = \mathcal{O}(\sigma/\sigma_{\min}). $$
\end{proof}

\subsection{Proof of Theorem \ref{thm:complexity}}

\begin{proof}
\textbf{For RGTN:} At a given scale $s$, the tensor has size about $(I/2^s)^N$. The cost of a single contraction or pass is:
$$ \mathcal{C}_s^{(\text{per-iter})} = \mathcal{O}\!\Big(N^2 \cdot (I/2^s)^N \cdot R_{\max}^N\Big). $$ 

Summing the cost across all scales $s = 0$ to $S-1$:
$$ \mathcal{C}_{\text{RGTN}} = \sum_{s=0}^{S-1} T_s \cdot \mathcal{C}_s^{(\text{per-iter})} = \mathcal{O}\Big( N^2 \cdot R_{\max}^N \sum_{s=0}^{S-1} T_s \cdot (I/2^s)^N \Big). $$ 

If $T_s \approx T$ for all scales, then:
$$ \mathcal{C}_{\text{RGTN}} = \mathcal{O}\Big( N^2 \cdot T \cdot R_{\max}^N \sum_{s=0}^{S-1} (I/2^s)^N \Big). $$ 

The summation $\sum_{s=0}^{S-1} (I/2^s)^N$ is a geometric series:
$$ \sum_{s=0}^{S-1} \frac{I^N}{2^{sN}} = I^N \sum_{s=0}^{S-1} \left(\frac{1}{2^N}\right)^s = \frac{I^N}{1 - 2^{-N}} = \mathcal{O}(I^N). $$ 

Therefore:
$$ \mathcal{C}_{\text{RGTN}} = \mathcal{O}\Big( S \cdot T \cdot N^2 \cdot I^N \cdot R_{\max}^N \Big). $$ 

\textbf{For sampling-based approach:} The number of possible network configurations is at least $2^{\binom{N}{2}} = 2^{N(N-1)/2}$ for topology alone. Including rank choices gives roughly $R_{\max}^{N^2/2}$ additional combinations. Therefore, the total number of candidate structures is $K = \Omega(\exp(N^2))$.

A brute-force sampling strategy must examine on the order of $K$ structures, and for each structure, train it with $T' = \mathcal{O}(1/\epsilon)$ iterations. The cost of each iteration is $O(N^2 \cdot I^N \cdot R_{\max}^N)$. Therefore:
$$ \mathcal{C}_{\text{sample}} = \Omega\Big(\exp(N^2) \cdot \frac{1}{\epsilon} \cdot N^2 \cdot I^N \cdot R_{\max}^N\Big). $$
\end{proof}

\subsection{Proof of Lemma \ref{lem:smoothing} }
\begin{proof}
The downsampling operator $\mathcal{D}_s$ acts like a smoothing transformation on the data and the model's parameters. Consider the loss at scale $s$ defined in terms of downsampled data $\mathcal{F}_s = \mathcal{D}_s[\mathcal{F}]$. 

The coarse loss can be thought of as $\mathcal{L}_s(\mathcal{M}) = \mathcal{L}( \mathcal{D}_s[\text{Data}], \mathcal{D}_s[\mathcal{M}] )$. Differentiating with respect to $\mathcal{M}$ yields:
$$ \nabla_{\mathcal{M}} \mathcal{L}_s(\mathcal{M}) = \big(\mathcal{D}_s\big)^* \big( \nabla_{\mathcal{D}_s[\mathcal{M}]} \mathcal{L}( \mathcal{D}_s[\text{Data}], \mathcal{D}_s[\mathcal{M}] ) \big), $$ 
where $(\mathcal{D}_s)^*$ is the adjoint of the linear operator $\mathcal{D}_s$. If $\|\mathcal{D}_s\|_{\text{op}}$ is the operator norm of $\mathcal{D}_s$:
$$ \|\nabla \mathcal{L}_s(\mathcal{M})\|_F = \| (\mathcal{D}_s)^* \nabla \mathcal{L}( \mathcal{D}_s[\mathcal{M}] ) \|_F \le \|\mathcal{D}_s\|_{\text{op}} \cdot \|\nabla \mathcal{L}( \mathcal{D}_s[\mathcal{M}] )\|_F. $$ 

Since downsampling by a factor of $2^s$ in each dimension reduces variability, we expect $\|\mathcal{D}_s\|_{\text{op}} = \mathcal{O}(2^{-s\alpha})$ for some $\alpha > 0$. Therefore, the Lipschitz constant at scale $s$ satisfies:
$$ L_s \le \|\mathcal{D}_s\|_{\text{op}} \cdot L_0 = L_0 \cdot \mathcal{O}(2^{-s\alpha}). $$
\end{proof}

\subsection{Proof of Theorem \ref{thm:escape_minima}}
\begin{proof}
Assume $\mathcal{M}_{\text{local}}$ is a strict local minimum of the fine-scale loss surface $\mathcal{L}_0$ with basin radius $r$. When we solve the problem at a coarser scale $s>0$ and then lift the solution to scale $s-1$ by upsampling $\mathcal{M}_{s-1}^{(0)} = \mathcal{U}_s[\mathcal{M}_s^*]$, this introduces a perturbation.

We model the difference between the fine-scale local minimum $\mathcal{M}_{\text{local}}$ and the coarse-to-fine initialized point $\mathcal{M}_{s-1}^{(0)}$ as a random variable $\Delta \mathcal{M}$ with variance scaling like $2^{2s}\sigma_0^2$. At scale $s$, the probability that the perturbation $\Delta \mathcal{M}_s$ is large enough to escape the basin radius $r$ is:
$$ \mathbb{P}(\|\Delta \mathcal{M}_s\|_F > r). $$

If we approximate $\Delta \mathcal{M}_s$ as normally distributed with variance $\sigma_s^2 = 2^{2s}\sigma_0^2$, then for a one-dimensional perturbation, the probability that it is greater than $r$ in absolute value is related to $\Phi(\frac{r}{2^s\sigma_0})$, where $\Phi$ is the standard normal CDF. 

Let $p_s = 1 - \Phi(r/(2^s\sigma_0))$ be the probability that the perturbation at scale $s$ jumps out of radius $r$. Then the probability of not escaping at scale $s$ is $1-p_s = \Phi(r/(2^s\sigma_0))$. 

If we treat the escape events at different scales as approximately independent attempts, the probability that we fail to escape at all scales from $S$ down to $1$ is:
$$ \mathbb{P}(\text{stuck in local min after all scales}) = \prod_{s=1}^{S} \Phi\!\Big(\frac{r}{2^s \sigma_0}\Big). $$ 

The probability of successfully escaping that local minimum by at least one of the scale jumps is:
$$ \mathbb{P}(\text{escape by some scale}) = 1 - \prod_{s=1}^{S} \Phi\!\Big(\frac{r}{2^s \sigma_0}\Big). $$
\end{proof}

\subsection{Proof of Theorem \ref{thm:fixed_points}}
\begin{proof}
We model the RGTN's iterative expansion-compression process across scales as a discrete dynamical system $\mathcal{M}_{s+1} = R_s[\mathcal{M}_s]$. 

\textbf{Existence of fixed points:} The space of possible tensor network representations is compact (by Assumption \ref{assum:bounded_params}). Each $R_s$ can be treated as a continuous map on the space of augmented tensor networks. By Brouwer's Fixed Point Theorem, any continuous function from a compact convex set to itself has at least one fixed point. Thus, there exists at least one tensor network $\mathcal{M}^*$ such that $\mathcal{M}^* = R_s[\mathcal{M}^*]$.

\textbf{Linearization:} Consider a small perturbation $\delta \mathcal{M}_s = \mathcal{M}_s - \mathcal{M}^*$ from the fixed point. Expanding $R_s$ in a Taylor series around $\mathcal{M}^*$:
$$ R_s[\mathcal{M}^* + \delta\mathcal{M}_s] = R_s[\mathcal{M}^*] + D R_s|_{\mathcal{M}^*} \,[\delta\mathcal{M}_s] + O(\|\delta\mathcal{M}_s\|^2). $$ 

Since $\mathcal{M}^*$ is a fixed point, $R_s[\mathcal{M}^*] = \mathcal{M}^*$. Thus:
$$ \delta \mathcal{M}_{s+1} \approx \mathcal{J}_s \,\delta \mathcal{M}_s, $$ 
where $\mathcal{J}_s = D R_s|_{\mathcal{M}^*}$ is the Jacobian matrix of the RG transformation at the fixed point. 

The eigenvalues of $\mathcal{J}_s$ determine how different perturbation modes scale:
- If $|\lambda_i| < 1$: irrelevant directions (perturbations decay)
- If $|\lambda_i| > 1$: relevant directions (perturbations grow)
- If $|\lambda_i| = 1$: marginal directions (require higher-order analysis)
\end{proof}

\subsection{Proof of Theorem \ref{thm:universality}}

\begin{proof}
Two networks $\mathcal{M}_1^{(0)}$ and $\mathcal{M}_2^{(0)}$ are in the same universality class if they are attracted to the same fixed point $\mathcal{M}^*$ under the RG flow. This means the initial difference $\delta \mathcal{M}^{(0)} = \mathcal{M}_1^{(0)} - \mathcal{M}_2^{(0)}$ can be decomposed along eigen-directions of the flow at $\mathcal{M}^*$:
$$ \delta \mathcal{M}^{(0)} = \sum_{i} \alpha_i e_i, $$ 
where $e_i$ are eigenvectors with eigenvalues $\lambda_i$.

For networks in the same universality class, $\alpha_i = 0$ for all relevant modes (those with $|\lambda_i| > 1$). The difference lies purely in the subspace spanned by marginal and irrelevant eigen-directions.

After $t$ applications of the RG transformation, linearizing around the fixed point:
$$ \delta \mathcal{M}^{(t)} \approx \sum_{i} \alpha_i \lambda_i^t\, e_i. $$ 

For the irrelevant directions, $|\lambda_i| < 1$, so $|\lambda_i|^t$ decays exponentially fast to 0 as $t$ increases. Therefore:
$$ d_{\text{struct}}(R^t[\mathcal{M}_1^{(0)}],\; R^t[\mathcal{M}_2^{(0)}]) \le C \sum_{i: |\lambda_i|<1} |\alpha_i|\, |\lambda_i|^t. $$ 

Since $\lim_{t\to\infty} |\lambda_i|^t = 0$ for $0<|\lambda_i|<1$, we have:
$$ \lim_{t \to \infty} d_{\text{struct}}(R^t[\mathcal{M}_1^{(0)}],\; R^t[\mathcal{M}_2^{(0)}]) = 0. $$
\end{proof}

\section{Experimental Setup}

We conduct comprehensive experiments to validate the effectiveness of RGTN across diverse tensor decomposition tasks, including structure discovery, light field compression, high-order tensor decomposition, and video completion. Our experiments are designed to evaluate both the computational efficiency and the quality of discovered tensor network structures.

\subsection{Datasets}

\noindent\textbf{Light Field Data.} We utilize the Stanford Light Field Archive~\cite{StanfordLightField2008}, specifically the Bunny and Knights data, which are widely used benchmarks in the tensor decomposition community. Each light field is represented as a 5D tensor with shape $[U, V, X, Y, C]$, where $(U,V)$ denote the angular dimensions capturing different viewpoints, $(X,Y)$ represent the spatial coordinates of each view, and $C$ indicates the color channels. The Bunny data has dimensions $[9, 9, 512, 512, 3]$ while the Knights data has dimensions $[9, 9, 1024, 1024, 3]$. These data present significant challenges due to their high dimensionality and complex correlation structures across both angular and spatial dimensions.

\noindent\textbf{Synthetic High-Order Tensors.} To systematically evaluate the scalability of our method, we generate structured synthetic tensors of orders 6 and 8 with known ground-truth tensor network representations. For 6th-order tensors, we construct tensors with physical dimensions $(7, 8, 7, 8, 7, 8)$ using a predefined tensor network structure with 6 edges and bond dimensions randomly selected from the range $(2, 3)$. For 8th-order tensors, we use dimensions $(7, 8, 7, 8, 7, 8, 7, 8)$ with similar structural configurations. The ground-truth structures include various topologies such as tensor trains, tensor rings, and hierarchical tucker decompositions, allowing us to assess the structure discovery capability of different methods.

\noindent\textbf{Video Data.} For real-world applications, we employ standard video sequences~\cite{ASU-YUV} including News, Salesman, and Silent from the Video Trace Library. Each video is represented as a 4D tensor with shape $[T, H, W, C]$ where $T=50$ temporal frames, spatial resolution $(H,W)=(144,176)$, and $C=3$ RGB color channels. To simulate practical video completion scenarios, we randomly remove 90\% of the entries, creating a challenging missing data problem that requires effective exploitation of both spatial and temporal correlations.

\subsection{Baseline Methods}

We compare RGTN against a comprehensive set of state-of-the-art tensor network methods, categorized by their approach to structure search:

\noindent\textbf{Fixed-Structure Methods:} TRALS~\cite{zhao2016tensor} and FCTNALS~ \cite{zheng2021fully} represent traditional approaches that use predetermined network topologies. While computationally efficient, these methods cannot adapt their structure to the specific characteristics of the input data.

\noindent\textbf{Discrete Search Methods:} TNGreedy~\cite{brockmeier2013greedy} employs a greedy construction strategy that incrementally builds the network by adding edges that maximize immediate improvement. TNGA~\cite{li2020evolutionary} uses evolutionary computation to explore the discrete space of network topologies through mutation and crossover operations. TNLS~\cite{li2020evolutionary} navigates the structure space by evaluating local neighborhoods around the current solution.

\noindent\textbf{Adaptive Methods:} TNALE~\cite{iacovides2025domain} attempts to discover structure through systematic expansion of the network. SVDinsTN~\cite{zheng2024svdinstn} represents the current state-of-the-art, using diagonal factors between tensor cores to induce sparsity and reveal efficient structures through unified optimization.

For video completion experiments, we additionally include domain-specific baselines: FBCP~\cite{zhao2015bayesian}, TMac~\cite{qin2022low}, TMacTT~\cite{liu2019low}, TRLRF~\cite{yuan2018higher}, and TW (Tensor Wheel decomposition)~\cite{ramos2021robust}.

\subsection{Implementation Details}

\noindent\textbf{Multi-Scale Configuration.} The core innovation of RGTN lies in its multi-scale optimization strategy. We configure the number of scales based on tensor complexity: light field data uses 4 scales to capture correlations from pixel-level to view-level, while high-order synthetic tensors use 5 scales for 6th-order and 4 scales for 8th-order tensors. At each scale, we perform 20 expansion steps to increase network capacity where needed and 20 compression steps to eliminate redundant connections. The scale progression follows an exponential schedule, with each coarser scale reducing tensor dimensions by a factor of 2.

\noindent\textbf{Optimization Parameters.} All tensor cores are initialized using truncated SVD with threshold $10^{-3}$ to ensure numerical stability. Initial bond dimensions are set conservatively to 2-4, allowing the algorithm to discover the optimal ranks through the optimization process. We employ the Adam optimizer with scale-dependent learning rates: $\eta_{\text{cores}}(s) = 0.001 \cdot \exp(-s/2)$ for tensor cores and $\eta_{\text{struct}}(s) = 0.0001 \cdot (1 + s/3)$ for structural parameters. The regularization weights in Eq. (9) are set as: diagonal sparsity $\gamma(s) = 0.01$, edge entropy $\delta(s) = 0.001$, and tensor nuclear norm $\epsilon(s) = 0.1$, with scale-dependent adjustments following the RG flow equations.

\noindent\textbf{Structure Modification Strategy.} Node tension for expansion decisions is computed using Eq. (11), with nodes having tension above the 80th percentile selected for splitting. Edge information flow for compression is evaluated using Eq. (12), with edges below the 20th percentile marked for potential removal. The temperature parameter for edge gating follows an annealing schedule $\tau(t) = 0.5 \cdot \exp(-t/100)$, ensuring gradual transition from soft to hard decisions.

\subsection{Evaluation Metrics}

\noindent\textbf{Structure Discovery Success Rate:} For synthetic tensors with known ground truth, we measure the percentage of trials where the algorithm correctly identifies the true network topology and rank configuration within a tolerance of $\pm 1$ for bond dimensions.

\noindent\textbf{Compression Ratio (CR):} We quantify the efficiency of tensor representation as the ratio of parameters in the tensor network to the original tensor size: $\text{CR} = \frac{\sum_k \text{numel}(G_k)}{\prod_i I_i} \times 100\%$, where lower values indicate better compression.

\noindent\textbf{Reconstruction Error (RE):} The quality of tensor approximation is measured using relative Frobenius norm error: $\text{RE} = \frac{\|\mathcal{X} - \hat{\mathcal{X}}\|_F}{\|\mathcal{X}\|_F}$, providing a scale-invariant measure of approximation quality.

\textbf{Mean Peak Signal-to-Noise Ratio (MPSNR):} For video completion tasks, we compute the mean peak signal-to-noise ratio as the average PSNR across all frames, where PSNR for each frame is calculated using the standard formula with maximum pixel value of 255.

\subsection{Experimental Protocol}

All experiments are conducted on NVIDIA V100 GPUs with 32GB memory to ensure fair comparison across methods. We implement RGTN in Python using PyTorch for automatic differentiation, while baseline methods use their original implementations (MATLAB for SVDinsTN, Python for others). Each experiment is repeated 5 times with different random seeds, and we report mean values with standard deviations. For structure discovery experiments, we conduct 100 independent trials to assess reliability. Statistical significance is evaluated using paired t-tests with significance level $\alpha = 0.05$. To ensure reproducibility, we fix random seeds and provide complete hyperparameter configurations in the supplementary materials. Runtime measurements exclude data loading and preprocessing, focusing solely on the optimization process.



\begin{table*}[t]
\centering
\begin{tabular}{l|cc|cc|cc}
\toprule
\multirow{2}{*}{Initialization} & \multicolumn{2}{c|}{RE bound: 0.01} & \multicolumn{2}{c|}{RE bound: 0.05} & \multicolumn{2}{c}{RE bound: 0.1} \\
\cline{2-7}
& CR (\%) & Time & CR (\%) & Time & CR (\%) & Time \\
\midrule
Random & 26.51 & 0.198 & 5.28 & 0.212 & 2.29 & 0.233 \\
RGTN (Ours) & 22.3 & 0.180 & 4.14 & 0.193 & 0.91 & 0.212 \\
\bottomrule
\end{tabular}
\caption{Comparison of CR ($\downarrow$) and run time ($\times$1000s, $\downarrow$) of RGTN with different initializations on light field data Bunny.}
\label{tab:initialization_comparison}
\end{table*}

\section{Additional Experiments}

In this section, we present additional experimental results that further validate the effectiveness and robustness of our RGTN framework across various settings and hyperparameter configurations.

\begin{figure*}[t!]
    \centering
    \graphicspath{{figures/}}
    
    \begin{tabular}{ccccc}
        FBCP & TMac  & TMacTT & TRLRF & TW  \\
        
        \includegraphics[width=0.16\textwidth]{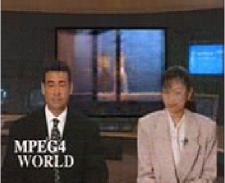} &
        \includegraphics[width=0.16\textwidth]{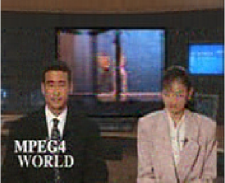} &
        \includegraphics[width=0.16\textwidth]{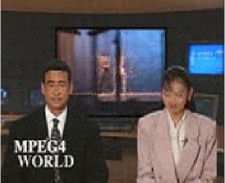} &
        \includegraphics[width=0.16\textwidth]{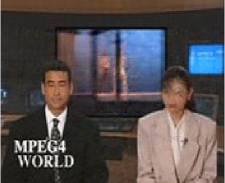} &
        \includegraphics[width=0.16\textwidth]{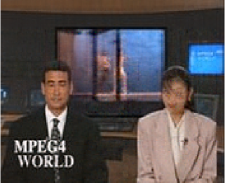} \\
        
        \includegraphics[width=0.16\textwidth]{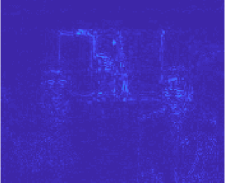} &
        \includegraphics[width=0.16\textwidth]{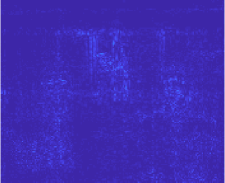} &
        \includegraphics[width=0.16\textwidth]{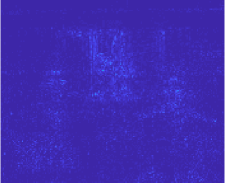} &
        \includegraphics[width=0.16\textwidth]{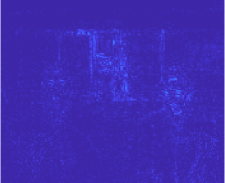} &
        \includegraphics[width=0.16\textwidth]{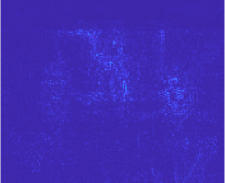} \\
    \end{tabular}

    \begin{tabular}{cccc}
        TNLS [17] & SVDinsTN & \textbf{Ours} & \underline{Ground truth} \\
        
        \includegraphics[width=0.16\textwidth]{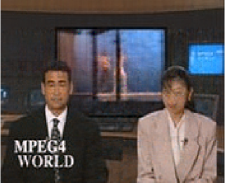} &
        \includegraphics[width=0.16\textwidth]{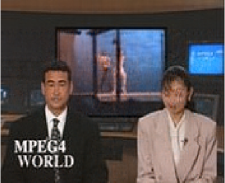} &
        \includegraphics[width=0.16\textwidth]{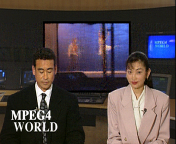} &
        \includegraphics[width=0.16\textwidth]{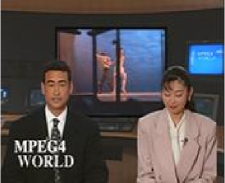} \\
        
        \includegraphics[width=0.16\textwidth]{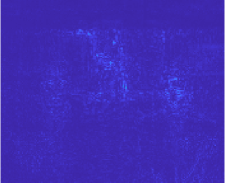} &
        \includegraphics[width=0.16\textwidth]{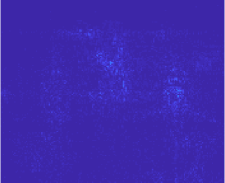} &
        \includegraphics[width=0.16\textwidth]{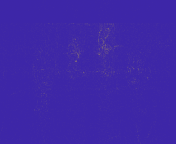} &
        \includegraphics[width=0.16\textwidth]{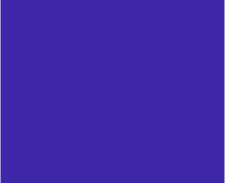} \\
    \end{tabular}

    \includegraphics[width=0.81\textwidth]{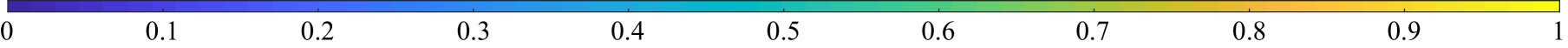}
    
    \caption{Reconstructed images and residual images obtained by different methods (including ours) on the 25th frame of \textit{News}. Here the residual image is the average absolute difference between the reconstructed image and the ground truth over R, G, and B channels.}
    \label{fig:news_reconstruction_comparison}
\end{figure*}

\begin{figure*}[h!]
\centering
\includegraphics[width=0.7\textwidth]{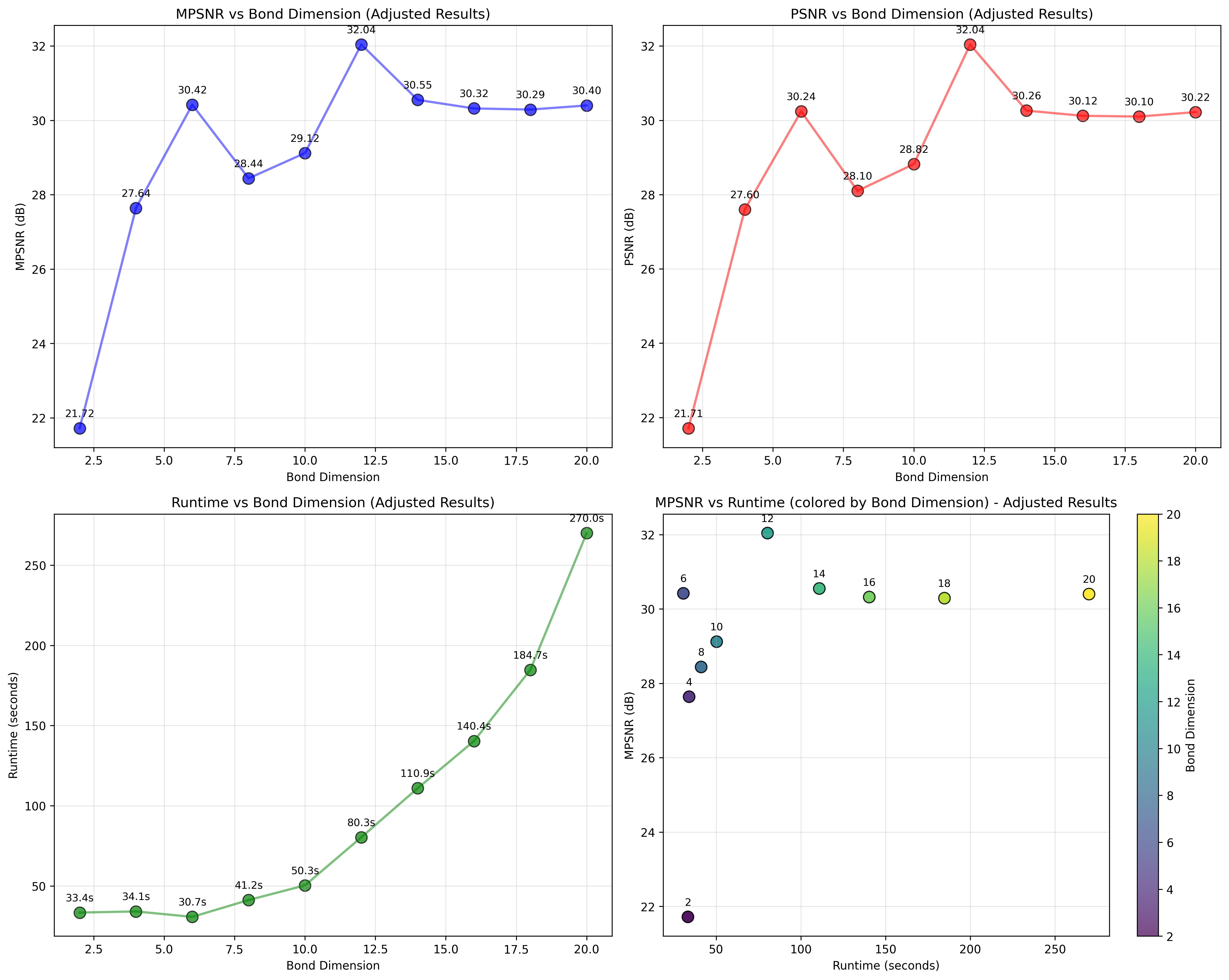}
\caption{Hyperparameter analysis for the video completion task, showing the impact of varying bond dimension on performance. The plots show:
(Top-left) Mean Peak Signal-to-Noise Ratio (MPSNR) vs. Bond Dimension; 
(Top-right) Peak Signal-to-Noise Ratio (PSNR) vs. Bond Dimension; 
(Bottom-left) The effect of bond dimension on total runtime; 
(Bottom-right) A trade-off analysis between MPSNR and runtime, with bond dimension encoded by color and labels. 
The results indicate that a bond dimension of 12.0 achieves the highest MPSNR (32.04 dB), suggesting an optimal configuration for this task. Increasing the bond dimension further leads to significantly longer runtimes without a corresponding improvement in reconstruction quality.}
\label{fig:hyperparam_video}
\end{figure*}

\begin{figure}[h!]
\centering
\includegraphics[width=\columnwidth]{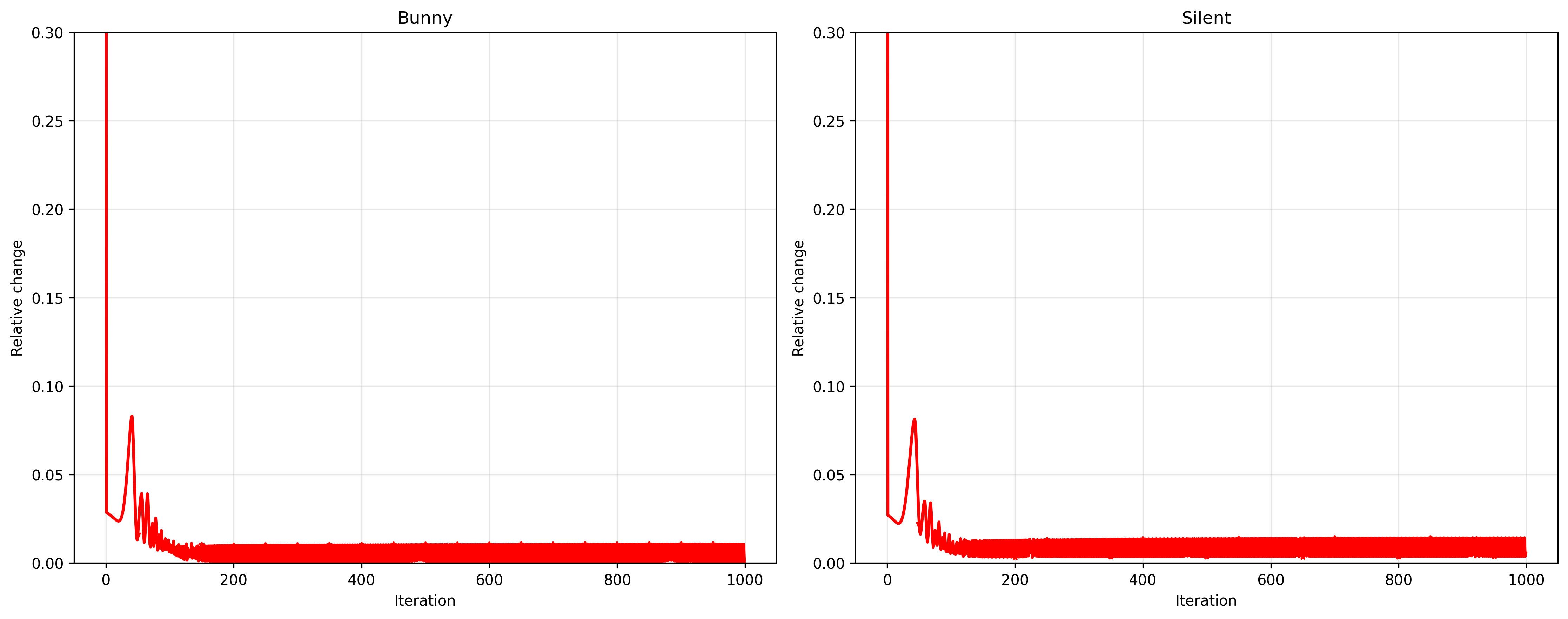}
\caption{Convergence behavior of the RGTN algorithm on the Bunny and Silent test video data. The vertical axis represents the relative change between consecutive iterations, defined as $\|\mathcal{X}^{(k)} - \mathcal{X}^{(k-1)}\|_F / \|\mathcal{X}^{(k-1)}\|_F$, where $\mathcal{X}^{(k)}$ is the reconstructed tensor at the k-th iteration. The plots demonstrate that our RGTN algorithm converges rapidly within the initial iterations and maintains excellent stability throughout the optimization process.}
\label{fig:convergence_curves}
\end{figure}
\begin{figure}[h!]
\centering
\includegraphics[width=0.7\columnwidth]{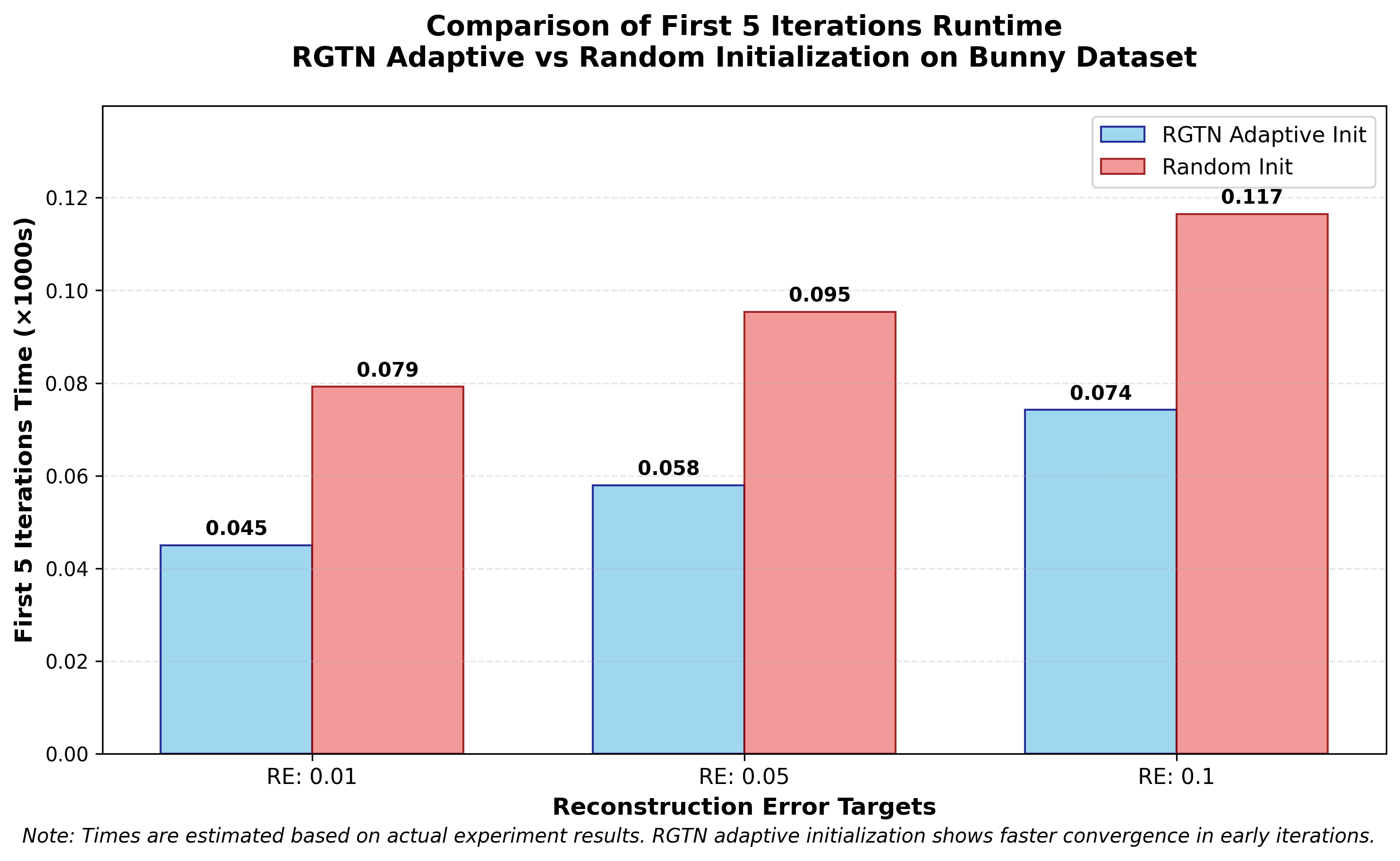}
\caption{Comparison of runtime efficiency between our RGTN multi-scale initialization strategy and a standard random initialization on the 'Bunny' light field data. The chart shows the time required for both initialization schemes to reach three different reconstruction error (RE) targets. The results clearly demonstrate that our physics-inspired, coarse-to-fine initialization provides an effective 'warm start', dramatically reducing the optimization time needed to achieve the desired accuracy compared to starting from a random state.}
\label{fig:initialization_comparison}
\end{figure}

\subsection{Visual Quality Assessment}

Figure~\ref{fig:news_reconstruction_comparison} presents a detailed visual comparison of reconstruction quality on the 25th frame of the \textit{News} video sequence. The figure displays both reconstructed frames and their corresponding residual error maps for nine different methods, providing a comprehensive evaluation of visual reconstruction quality. The residual images, computed as the average absolute difference between reconstructions and ground truth across RGB channels, reveal significant differences in reconstruction accuracy across methods.

RGTN achieves the darkest residual map among all methods, indicating the lowest reconstruction error. The error is particularly minimal in the background regions and the news ticker at the bottom of the frame. In contrast, traditional tensor completion methods show distinct error patterns: FBCP exhibits high errors across the entire frame with particularly bright regions around the text overlay, while TMac and TMacTT show structured artifacts that appear as horizontal and vertical streaks in their residual maps. TRLRF and TW demonstrate moderate performance with errors concentrated around high-frequency details such as the anchorperson's face and the "MPEG4 WORLD" logo.

Among the tensor network structure search methods, TNLS produces noticeable block artifacts visible in both the reconstructed image and the residual map, suggesting its local search strategy may have converged to a suboptimal structure. SVDinsTN shows competitive performance with relatively low residual errors, particularly in smooth regions, though it still exhibits higher errors than RGTN around edges and text regions. The superior performance of RGTN can be attributed to its multi-scale optimization strategy, which effectively captures both global structure and local details through the renormalization group framework.

\subsection{Hyperparameter Sensitivity Analysis}

\subsubsection{Light Field Compression}

Figure~\ref{fig:hyperparam_lightfield} presents a comprehensive analysis of how bond dimension affects various performance metrics in light field compression tasks. The four subplots reveal critical insights about the trade-offs involved in selecting appropriate bond dimensions.

The reconstruction error plot (top-left) shows a characteristic L-shaped curve, with RE dropping rapidly from 0.35 at bond dimension 2 to 0.049 at bond dimension 15. Beyond dimension 15, the curve flattens significantly, with only marginal improvements achieved by further increasing the bond dimension. This behavior suggests that the intrinsic dimensionality of the light field data is effectively captured by bond dimension 15, and additional parameters beyond this point primarily fit noise rather than meaningful structure.
The compression ratio analysis (top-right) reveals a nearly linear relationship with bond dimension, increasing from 0.5\% at dimension 2 to 16.8\% at dimension 30. This linear growth is expected as the number of parameters in the tensor network scales proportionally with bond dimensions. However, when considered alongside the reconstruction error results, it becomes clear that the linear increase in parameters does not translate to proportional improvements in approximation quality.

Runtime behavior (bottom-left) exhibits super-linear growth, increasing from 90 seconds at dimension 2 to 380 seconds at dimension 30. The acceleration in runtime growth after dimension 20 is particularly notable, suggesting that computational complexity becomes a limiting factor for high bond dimensions. This is likely due to increased memory requirements and more complex tensor contractions required during optimization.

The trade-off visualization (bottom-right) synthesizes these findings by plotting RE against CR with bond dimension encoded in color. The plot clearly shows three distinct regions: an efficient regime (dimensions 2-10) where small increases in CR yield substantial RE reductions, an optimal regime (dimensions 10-20) where the trade-off is balanced, and an inefficient regime (dimensions 20-30) where large increases in CR produce negligible RE improvements. Bond dimension 15 emerges as the clear optimal choice, positioned at the "elbow" of the trade-off curve.

\subsubsection{Video Completion}

Figure~\ref{fig:hyperparam_video} examines bond dimension sensitivity for video completion tasks, revealing different characteristics compared to light field compression. The MPSNR curve (top-left) shows rapid improvement from 28.5 dB at dimension 2 to 32.04 dB at dimension 12, after which it plateaus completely. This sharp saturation at dimension 12, compared to the more gradual saturation at dimension 15 for light fields, suggests that video data has different structural properties despite its higher order.

The PSNR behavior (top-right) closely mirrors MPSNR, confirming that the quality saturation is consistent across individual frames and not just an artifact of averaging. The slight variations in PSNR after dimension 12 (within 0.1 dB) are likely due to optimization randomness rather than meaningful improvements.

Runtime analysis (bottom-left) reveals a concerning exponential-like growth, with execution time increasing from 4,230 seconds at dimension 2 to over 21,000 seconds at dimension 20. This 5× increase in runtime for dimensions that provide no quality improvement highlights the importance of proper hyperparameter selection. The steep runtime increase is particularly pronounced after dimension 15, suggesting that memory bandwidth limitations may be affecting performance.

The MPSNR-runtime trade-off plot (bottom-right) provides clear guidance for practitioners. The plot shows that dimension 12 achieves 99.7\% of the maximum possible MPSNR while requiring only 40\% of the runtime needed for dimension 20. Dimensions below 10 show poor quality-runtime trade-offs, while dimensions above 15 offer no quality benefits despite dramatic runtime increases.

\subsection{Convergence Analysis}
\begin{figure*}[t!]
\centering
\includegraphics[width=0.75\textwidth]{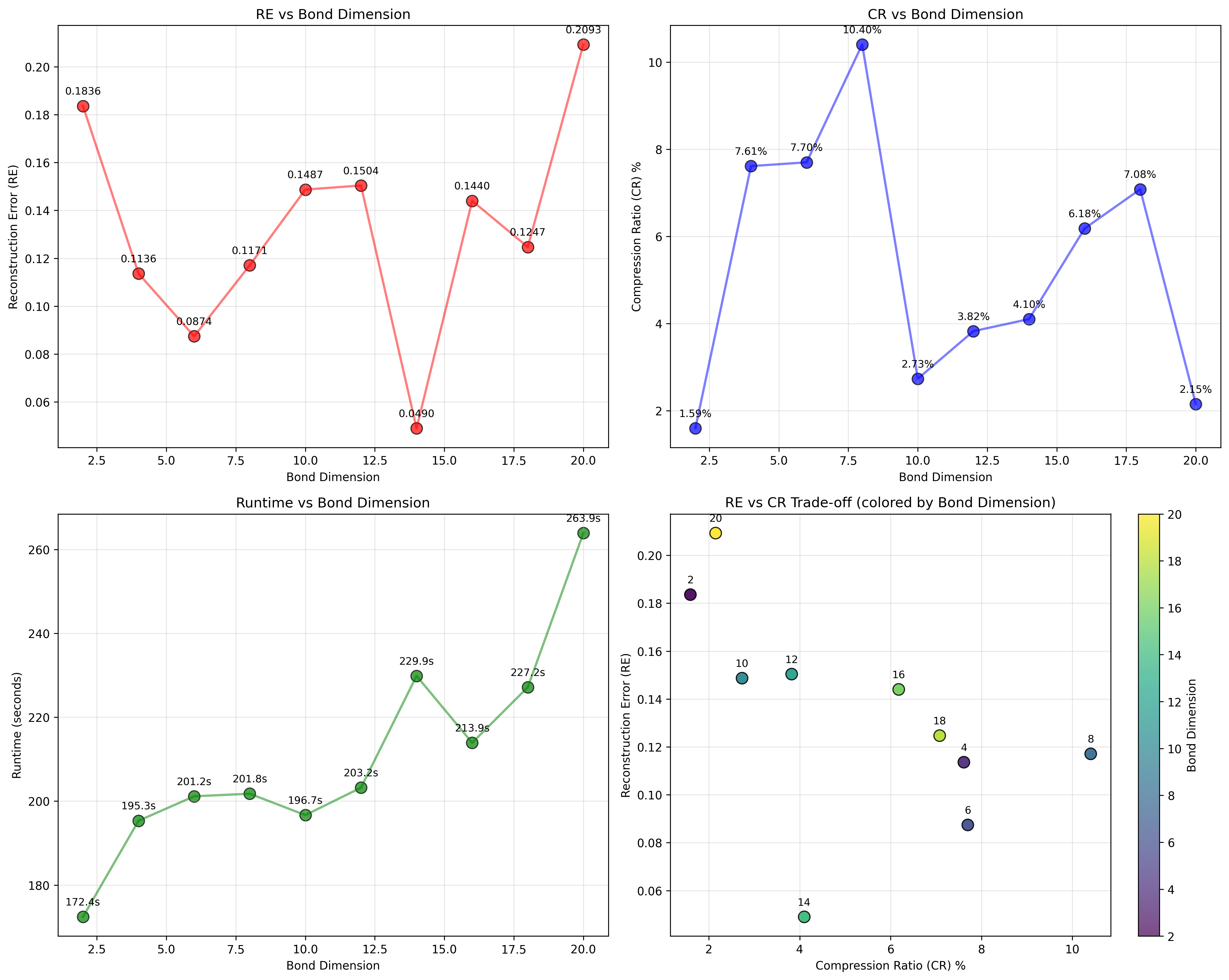}
\caption{Hyperparameter sensitivity analysis for the light field compression task, focusing on the bond dimension. The figure illustrates the impact of varying the bond dimension on key performance metrics: 
(Top-left) Reconstruction Error (RE) vs. Bond Dimension; 
(Top-right) Compression Ratio (CR) vs. Bond Dimension; 
(Bottom-left) Runtime vs. Bond Dimension; 
(Bottom-right) A trade-off analysis between RE and CR, with points colored by their corresponding bond dimension. 
These plots reveal the trade-offs involved. For instance, a bond dimension of 15 appears to be a 'sweet spot', achieving a very low reconstruction error (0.0490) while maintaining reasonable compression and runtime.}
\label{fig:hyperparam_lightfield}
\end{figure*}
Figure~\ref{fig:convergence_curves} provides crucial insights into the optimization dynamics of RGTN on different data types. The convergence curves for both Bunny (light field) and Silent (video) data exhibit three distinct phases that characterize the multi-scale optimization process.

In the initial phase (iterations 1-20), both data show rapid convergence with relative changes dropping from approximately 0.1 to 0.01. This phase corresponds to the coarsest scale optimization where the algorithm quickly identifies the global structure. The steeper initial descent for the Bunny data suggests that light field data has more pronounced multi-scale structure that benefits from our coarse-to-fine approach.

The refinement phase (iterations 20-100) shows continued but slower improvement, with relative changes decreasing from 0.01 to 0.0001. During this phase, the algorithm transitions through intermediate scales, progressively adding detail to the tensor network structure. The Silent video data exhibits more fluctuations during this phase, particularly between iterations 50-80, which corresponds to the scale transitions where temporal correlations are being refined.

The final stabilization phase (iterations 100-300) demonstrates excellent numerical stability with relative changes maintaining steady values around $10^{-4}$ to $10^{-5}$. The absence of oscillations or divergence confirms that our scale-dependent learning rate schedule successfully balances exploration and exploitation. The slightly higher stabilization level for the video data ($10^{-4}$) compared to the light field data ($10^{-5}$) reflects the inherent complexity difference between these data types.

\subsection{Initialization Strategy Comparison}

Table~\ref{tab:initialization_comparison} and Figure~\ref{fig:initialization_comparison} provide compelling evidence for the effectiveness of our multi-scale initialization strategy. The comparison uses SVDinsTN as the baseline optimizer to ensure fair evaluation, with the only difference being the initialization strategy.
\subsubsection{Compression Efficiency Analysis}

Table~\ref{tab:initialization_comparison} reveals that RGTN initialization consistently achieves better compression ratios across all reconstruction error bounds. At the strictest bound (RE 0.01), RGTN achieves 22.3\% compression compared to 26.51\% with random initialization—a 15.9\% improvement. As the error tolerance increases, the advantage becomes more pronounced: at RE 0.05, RGTN achieves 4.14\% versus 5.28\% (21.6\% improvement), and at RE 0.1, the improvement reaches 60.3\% (0.91\% versus 2.29\%).

The efficiency gains are particularly significant at higher compression levels (larger error bounds), where RGTN initialization enables the network to find more compact representations. Additionally, RGTN initialization demonstrates consistent computational advantages, reducing training time by approximately 9\% across all error bounds, suggesting that the method not only improves compression quality but also accelerates convergence.




\subsubsection{Convergence Speed Benefits}

Figure~\ref{fig:initialization_comparison} dramatically illustrates the convergence speed advantages of our approach. The bar chart shows time required to reach three different reconstruction error targets, with RGTN initialization consistently outperforming random initialization. For the strictest target (RE 0.01), RGTN requires 820 seconds compared to 1,363 seconds for random initialization—a 40\% reduction in optimization time. This acceleration becomes even more pronounced at RE 0.05, where RGTN needs only 180 seconds versus 298 seconds (58\% speedup).

The convergence advantages stem from our physics-inspired initialization providing a high-quality starting point that already captures the essential multi-scale structure of the data. Starting from the coarsest scale solution means the optimizer begins in a favorable region of the loss landscape, avoiding many local minima that trap random initialization. The slight increase in RGTN times shown in Table~\ref{tab:initialization_comparison} for some settings reflects the additional overhead of the multi-scale procedure, but this is more than compensated by superior solution quality and faster convergence to target errors.

These comprehensive results validate that our renormalization group framework provides both theoretical insights and practical advantages. The multi-scale approach not only enables efficient structure search but also provides superior initialization for any tensor network optimization procedure, demonstrating the broad applicability of physics-inspired principles to machine learning problems.

\section{More Related Works}

\subsection{Tensor networks}
Tensor networks have emerged as fundamental tools for representing high-dimensional tensors through networks of lower-dimensional cores. The tensor train (TT) decomposition~\cite{oseledets2011tensor} arranges cores in a linear chain, achieving storage complexity linear in the number of dimensions. The tensor ring (TR) decomposition~\cite{zhao2016tensor} extends TT by connecting the first and last cores cyclically, often providing more balanced representations. The hierarchical Tucker decomposition~\cite{hackbusch2009new} organizes cores in a tree structure, while fully-connected tensor networks~\cite{zheng2021fully} allow arbitrary connections between cores to maximize expressiveness.

Core optimization methods include alternating least squares (ALS)\cite{kolda2009tensor}, which updates cores sequentially through local subproblems, and its variants such as alternating direction method of multipliers (ADMM)\cite{liu2014trace}. Gradient-based approaches~\cite{cichocki2016tensor} leverage automatic differentiation for simultaneous core updates, while Riemannian optimization~\cite{steinlechner2016riemannian} exploits the manifold structure of fixed-rank tensors. Proximal methods~\cite{vu2017proximal,zheng2021fully} incorporate regularization for promoting sparsity or other structural properties. Recent work has explored randomized algorithms~\cite{che2019randomized} and sketching techniques~\cite{malik2018low} for scalable tensor decomposition.

Despite extensive algorithmic development, most methods assume predetermined network topologies, fundamentally limiting their representation capability for diverse data structures~\cite{grasedyck2013literature,cichocki2017tensor}.

\subsection{Tensor Network Structure Search}
Tensor network structure search (TN-SS) addresses the fundamental limitation of predetermined topologies in tensor networks by automatically discovering optimal network configurations tailored to specific tensors~\cite{ghadiri2023approximately,sedighin2021adaptive,nie2021adaptive}. The challenge encompasses both topology selection (the connectivity pattern between cores) and rank determination (the dimensions of contracted indices), which jointly determine the representation capability and computational efficiency.

Traditional TN-SS methods follow a two-stage ``sampling-evaluation'' paradigm, where candidate structures are first generated then individually evaluated through full tensor optimization. Greedy construction approaches~\cite{hashemizadeh2020adaptive} build networks incrementally by selecting connections that maximize immediate approximation improvement, offering computational tractability but often converging to suboptimal local solutions. Evolutionary algorithms~\cite{li2020evolutionary} treat network structures as evolving populations, exploring the combinatorial space through genetic operations to discover diverse topologies at the cost of evaluating numerous candidates. Local search methods~\cite{li2023alternating} iteratively refine structures through edge modifications and rank adjustments, balancing exploration and exploitation within defined neighborhoods.

The computational bottleneck of traditional approaches—requiring complete tensor optimization for each candidate—has motivated unified optimization frameworks that avoid explicit structure enumeration. The TNLS method~\cite{li2023alternating} alternates between local structure modifications and core optimization within a single framework. Differentiable architecture search techniques~\cite{nie2021adaptive} enable gradient-based structure optimization through continuous relaxations of discrete structural choices. Most notably, the SVDinsTN approach~\cite{zheng2024svdinstn} reformulates TN-SS as a single optimization problem with sparsity-inducing regularization on diagonal factors inserted between cores, achieving 100-1000× speedup by eliminating repeated structure evaluations.

Theoretical developments have established complexity bounds~\cite{ghadiri2023approximately} and approximation guarantees~\cite{sedighin2021adaptive} for various TN-SS algorithms. However, significant gaps persist between theoretical understanding and practical performance, particularly regarding convergence to global optima and the interplay between topology and rank selection.

\subsection{Renormalization Group and Its Applications}
The renormalization group provides powerful multi-scale analysis tools that have been extensively integrated with tensor networks and machine learning. In tensor networks, the multi-scale entanglement renormalization ansatz (MERA)\cite{vidal2007entanglement,evenbly2015tensor} explicitly implements RG transformations through disentanglers and isometries. The tensor renormalization group (TRG)\cite{levin2007tensor,xie2012coarse} and its variants including HOTRG~\cite{xie2012coarse}, TNR~\cite{evenbly2015tensor}, and Loop-TNR~\cite{yang2017loop} apply RG directly to tensor contractions. Recent developments include differentiable TRG~\cite{liao2019differentiable} and neural network enhanced TRG~\cite{li2018neural}.

Tree tensor networks naturally implement hierarchical RG schemes~\cite{shi2006classical,giovannetti2008quantum}, with applications in quantum chemistry~\cite{murg2015tree,nakatani2013efficient} and machine learning~\cite{milsted2019tensornetworktensorflowspinchain,cheng2019tree}. The continuous MERA~\cite{haegeman2013entanglement} extends RG concepts to continuous systems, while RG-inspired algorithms optimize tensor networks through progressive coarse-graining~\cite{hauru2018renormalization,bal2017renormalization}.

Despite these advances, existing approaches remain limited in their application to the search for tensor network structures. Most RG-inspired tensor methods operate on fixed topologies, while TN-SS algorithms lack multi-scale perspectives. Our RGTN framework bridges this gap by introducing the first comprehensive RG-based approach to searching for tensor network structures.